\documentclass[journal]{IEEEtran}

\usepackage{xcolor,soul,framed} 

\colorlet{shadecolor}{yellow}
\usepackage[pdftex]{graphicx}
\graphicspath{{../pdf/}{../jpeg/}}
\DeclareGraphicsExtensions{.pdf,.jpeg,.png}

\usepackage[cmex10]{amsmath}
\usepackage{array}
\usepackage{mdwmath}
\usepackage{mdwtab}
\usepackage{eqparbox}
\usepackage{url}
\usepackage{amsthm}
\newtheorem{definition}{Definition}
\newtheorem{theorem}{Theorem}
\newtheorem{lemma}{Lemma}
\usepackage{amssymb}
\usepackage{amsmath, amssymb}
\usepackage{algorithm}
\usepackage{algpseudocode}
\usepackage{amsmath}
\usepackage{float}
\usepackage{multirow}
\usepackage{cite}
\usepackage{ragged2e}
\usepackage{xcolor}  
\usepackage[
    colorlinks=true,        
    linkcolor=red,          
    citecolor=blue,         
    urlcolor=blue,         
    filecolor=magenta       
]{hyperref}
\usepackage{soul}
\sethlcolor{yellow}
\usepackage{upgreek}
\usepackage{textcomp}

\begin{document}
\bstctlcite{IEEEexample:BSTcontrol}
    \title{M\textsuperscript{3}GCLR: Multi-View Mini-Max Infinite Skeleton-Data Game Contrastive Learning For Skeleton-Based Action Recognition}
  \author{Yanshan Li,
      Ke Ma,
      Miaomiao Wei,
      Linhui Dai$^{*}$

  \thanks{This work was partially supported by National Natural Science Foundation of China (No. 62471317), Natural Science Foundation of Shenzhen (No. JCYJ20240813141331042), and Guangdong Provincial Key Laboratory (Grant 2023B1212060076). }

  \thanks{Yanshan Li, Ke Ma, Miaomiao Wei, and Linhui Dai$^{*}$ are with Shenzhen University, the Institute of Intelligence Information Processing, and the Guangdong Key Laboratory of Intelligent Information Processing, Shenzhen 518000, China (emails: lys@szu.edu.cn; 2400042033@mails.szu.edu.cn; 1054772440@qq.com; dailinhui@szu.edu.cn$^{*}$).}
  }

\markboth{IEEE TRANSACTIONS ON PATTERN ANALYSIS AND MACHINE INTELLIGENCE, VOL.~xxx, NO.~xxx, xxx
}{Roberg \MakeLowercase{\textit{et al.}}: High-Efficiency Diode and Transistor Rectifiers}

\maketitle

\begin{abstract}
In recent years, contrastive learning has drawn significant attention as an effective approach to reducing reliance on labeled data. However, existing methods for self-supervised skeleton-based action recognition still face three major limitations: insufficient modeling of view discrepancies, lack of effective adversarial mechanisms, and uncontrollable augmentation perturbations. To tackle these issues, we propose the \hl{M}ulti-view \hl{M}ini-\hl{M}ax infinite skeleton-data \hl{G}ame \hl{C}ontrastive \hl{L}earning for skeleton-based action \hl{R}ecognition (M\textsuperscript{3}GCLR), a game-theoretic contrastive framework. First, we establish the \hl{I}nfinite \hl{S}keleton-data \hl{G}ame (ISG) model and the ISG equilibrium theorem, and further provide a rigorous proof, enabling mini-max optimization based on multi-view mutual information. Then, we generate normal–extreme data pairs through multi-view rotation augmentation and adopt temporally averaged input as a neutral anchor to achieve structural alignment, thereby explicitly characterizing perturbation strength. Next, leveraging the proposed equilibrium theorem, we construct a strongly adversarial mini-max skeleton-data game to encourage the model to mine richer action-discriminative information. Finally, we introduce the dual-loss equilibrium optimizer to optimize the game equilibrium, allowing the learning process to maximize action-relevant information while minimizing encoding redundancy, and we prove the equivalence between the proposed optimizer and the ISG model. Extensive Experiments show that M\textsuperscript{3}GCLR achieves three-stream 82.1\%, 85.8\% accuracy on NTU RGB+D 60 (X-Sub, X-View) and 72.3\%, 75.0\% accuracy on NTU RGB+D 120 (X-Sub, X-Set). On PKU-MMD Part I and II, it attains 89.1\%, 45.2\% in three-stream respectively, all results matching or outperforming state-of-the-art performance. Ablation studies confirm the effectiveness of each component.
\end{abstract}


%
\IEEEpeerreviewmaketitle


\section{Introduction}

\IEEEPARstart{S}{keleton-based} action recognition has recently attracted extensive attention from both academia and industry \cite{han2013enhanced, zheng2023deep}. To address the limitations of supervised learning methods that heavily rely on large-scale labeled data, self-supervised learning has gradually become a research hotspot in this field \cite{Su2019PredictCluster, Lin2020MS2L, Zheng2018Unsupervised, Thoker2021SkeletonContrastive}. As an important branch of self-supervised learning, contrastive learning has demonstrated great potential in skeleton-based action recognition owing to its intuitive learning mechanism and strong feature representation capability. It constructs positive and negative sample pairs, maximizing the similarity between different augmented views of the same sample (positive pairs) while minimizing the similarity across different samples (negative pairs), which drives the model to learn more discriminative representations \cite{He2019MoCo}.

Although existing self-supervised contrastive learning methods \cite{He2019MoCo, Guo2022AimCLR, Lin2023ActionletDependent, Gao2021ContrastiveSSL} have achieved remarkable progress in skeleton-based action recognition, their performance in multi-view self-supervised learning remains unsatisfactory. First, skeleton data are represented in the form of 3D joint coordinates, which are highly sensitive to camera viewpoints. Even slight changes in the observation angle may lead to significant variations in the recognition results. Second, current methods generally lack sufficient adversarial modeling, failing to effectively capture competitive and cooperative relationships in feature learning, directly limiting the upper bound of the representation capability. Third, data augmentation plays a key role in contrastive learning. Appropriate augmentation can encourage the model to explore more motion patterns while reducing interference caused by redundant information.

Strong adversarial learning is an effective way to break the upper bound of model capability. However, previous contrastive learning approaches \cite{He2019MoCo, Guo2022AimCLR, Lin2023ActionletDependent, Gao2021ContrastiveSSL, Guo2024AimCLRPlus} often overlook such adversarial modeling. Game theory serves as a powerful mathematical tool for modeling cooperation and competition among multiple rational and intelligent decision-makers or agents \cite{rapoport2012game}, where mini-max game represents the strongest form of adversarial interaction. Motivated by this, we adopt a mini-max game as the theoretical foundation and further extend it to adapt to the characteristics of skeleton data, proposing the \textbf{M}ulti-view \textbf{M}ini-\textbf{M}ax infinite skeleton-data \textbf{G}ame \textbf{C}ontrastive \textbf{L}earning for skeleton-based action \textbf{R}ecognition (\textbf{M\textsuperscript{3}GCLR}). First, we construct the \textbf{I}nfinite \textbf{S}keleton-data \textbf{G}ame (\textbf{ISG}) model and introduce the ISG equilibrium theorem together with a rigorous proof, establishing a solid theoretical foundation for M\textsuperscript{3}GCLR. Second, to address view dependency, we propose the \textbf{M}ulti-view \textbf{R}otation-based \textbf{A}ugmentation \textbf{M}odule (\textbf{MRAM}) that generates diverse augmented views by dynamically adjusting rotation angles between normal and extreme augmentations. Combined with the temporally averaged input as a neutral anchor, a contrastive triplet consisting of normal augmentation, extreme augmentation, and the averaged data is formed. This design simulates realistic viewpoint variations to enhance robustness while mitigating feature distortions caused by camera shifts through anchor-based alignment. Next, to alleviate the lack of adversarial modeling, we further analyze the game process with infinite strategy sets and propose the \textbf{M}utual-information-based \textbf{M}ini-\textbf{M}ax \textbf{I}nfinite \textbf{S}keleton-data \textbf{G}ame \textbf{M}odule (\textbf{M\textsuperscript{3}ISGM}) based on theoretical derivations. The normal and extreme augmentations act as opposing players, and mutual information serves as the payoff function to formulate the mini-max game. By maximizing the discrepancy between both enhanced views and the averaged data, the M\textsuperscript{3}ISGM explicitly models the competitive dynamics of feature learning, enhancing adversarial capability and pushing the representation learning beyond its existing limits. Finally, to further amplify action-discriminative features across different sequences, reduce redundant information under view variations of the same sequence, and ensure stable convergence of the M\textsuperscript{3}ISGM equilibrium, we propose the \textbf{D}ual-\textbf{L}oss-based \textbf{E}quilibrium \textbf{O}ptimizer (\textbf{DLEO}). By designing a loss function that averages the losses of normal and extreme augmentations, DLEO transforms ISG into a complementary optimization problem, enhancing feature discriminability while viewpoint robustness and adversarial effectiveness.

The major contributions of this paper are summarized as follows:

(1) Based on classical game theory for modeling strategic games and Nash equilibrium analysis, this paper proposes the \textbf{I}nfinite \textbf{S}keleton-data \textbf{G}ame (\textbf{ISG}) model and introduces an ISG equilibrium theorem with a rigorous proof, providing a solid theoretical foundation for the proposed method.

(2) In terms of multi-view data augmentation for skeleton sequences, to overcome insufficient modeling of viewpoint discrepancies, we propose a \textbf{M}ulti-view \textbf{R}otation-based \textbf{A}ugmentation \textbf{M}odule (\textbf{MRAM}). By applying multiple rotation matrices to generate diverse enhanced views, including normal- and extreme-augmented sequences, MRAM enriches the training data distribution and enables the model to better adapt to viewpoint variations.

(3) For optimizing encoder parameters under multi-view inputs, to address the lack of adversarial modeling, excessive redundant information, insufficient feature discrimination, and uncontrollable perturbations in existing methods, we construct the \textbf{M}utual-information-based \textbf{M}ini-\textbf{M}ax \textbf{I}nformation \textbf{S}keleton-data \textbf{G}ame \textbf{M}odule (\textbf{M\textsuperscript{3}ISGM}) based on the proposed ISG equilibrium theorem. Furthermore, we equivalently develop the \textbf{D}ual-\textbf{L}oss-based \textbf{E}quilibrium \textbf{O}ptimizer (\textbf{DLEO}) to reduce redundancy and enhance inter-sequence discriminability while ensuring stable convergence of the equilibrium, enabling the network to learn more representative and discriminative features. In addition, we provide a rigorous proof of the equivalence between DLEO and M\textsuperscript{3}ISGM. 

(4) Extensive experiments show that M\textsuperscript{3}GCLR consistently outperforms existing methods across benchmarks. On NTU RGB+D 60/120, it achieves 2–4\% accuracy gains under X-Sub and X-View/X-Set protocols, surpassing prior SOTA by up to 3\%. On PKU-MMD, M\textsuperscript{3}GCLR delivers competitive or superior performance, including a 4.0\% improvement on Part II. Our code is available at \url{https://github.com/Ixiaohuihuihui/M3GCLR}.

\section{Related Work on Self-supervised Skeleton-based Action Recognition}

Self-supervised skeleton-based action recognition has achieved remarkable progress in recent years by designing diverse spatio-temporal transformations and auxiliary tasks to learn more discriminative latent representations. In 2020, Lin et al. \cite{Lin2020MS2L} proposed MS2L, which integrates motion prediction and jigsaw recognition into a multi-task framework to enhance the representation capability of unlabeled skeleton data from both temporal dependency and motion semantics. In 2021, Rao et al. \cite{Rao2021MomentumLSTM} introduced AS-CAL, which employs multiple augmentation strategies to maintain feature stability under viewpoint perturbations. Li et al. \cite{Li2021CrossViewConsistency} proposed CrosSCLR, achieving cross-view consistent representations through multi-view complementary supervision. Guo et al. \cite{Guo2022AimCLR} introduced extreme augmentations, Energy-Attention-guided Dropout (EADM), and Dual Distribution Discrepancy Minimization (D3M) in AimCLR to significantly improve feature discriminability. Liu et al. \cite{Liu2021AdaptiveMultiView} presented AMW-GCN, which simultaneously models spatial structure and temporal dynamics using an adaptive multi-view transformation module and a multi-stream GCN architecture, thereby improving robustness to viewpoint variations. In 2022, Chen et al. \cite{Chen2022MixedSkeleton} adopted a hybrid augmentation framework combined with topological information to mine hard examples, while Xia et al. \cite{Xia2022LAGANet} proposed LAGA-Net, which enhances frame-level saliency via motion-guided channel attention and global attention. In 2023, Dong et al. \cite{Dong2023HiCLR} proposed HiCLR, employing stage-wise downsampling to achieve multi-scale feature fusion. In 2024, Wu et al. \cite{Wu2024SCDNet} introduced SCD-Net, which decouples spatial and temporal cues to separately model action semantics.

Multi-view modeling is a key factor affecting the performance of contrastive-learning-based skeleton action recognition. In 2019, He et al. \cite{He2019MoCo} proposed MoCo, which leverages a momentum encoder and FIFO queue to enable large-scale contrastive learning, establishing a general framework for skeleton representation learning. In 2021, Li et al. \cite{Li2020ShapeMotion} proposed STVIM, which incorporates rotor-based viewpoint transformation in a geometric algebra framework to achieve multi-stream consistency in shape and motion representations; Li et al. \cite{Li2021CrossViewConsistency} further enhanced consistency in contrastive objectives through multi-view complementary supervision in CrosSCLR. In 2023, Wang et al. \cite{Wang2023SS3D} introduced ASAR, which employs adaptive re-calibration to balance the contrastive sample distribution and mitigate feature collapse and gradient imbalance. Hua et al. \cite{Hua2023PartAwareCLR} proposed Part-Aware Contrastive Learning, which models hard sample distributions via human body partitioning to strengthen local feature discrimination. In 2024, Guo et al. \cite{Guo2024AimCLRPlus} introduced 3s-AimCLR++, which incorporates three-stream aggregation and multi-stream interaction to fully exploit cross-modal complementary information.

\begin{table*}[t]
\centering
\caption{Significant Mathematical Symbols in This Section}
\label{tab:symbols}
\renewcommand{\arraystretch}{1.25}

\begingroup
\setlength{\parindent}{0pt}

\begin{tabular}{|>{\centering\arraybackslash}m{3.7cm}|>{\justifying\arraybackslash}m{10.2cm}|}
\hline
\multicolumn{1}{|c|}{Symbol} & \multicolumn{1}{c|}{Description} \\
\hline
$\Delta(\Omega)$ &
\noindent The set of all probability distributions over the sample space $\Omega$. \\

$\times$ (or $\times_{i \in N}$) &
\noindent The Cartesian product of two or more sets, representing all possible ordered tuples. Example: if $A=\{1\}$ and $B=\{2,3\}$, then $A \times B = \{(1,2),(1,3)\}$. \\

$\Gamma = (N,(C_i)_{i\in N},(u_i)_{i\in N})$ &
\noindent A strategic-form game $\Gamma$, where $N$ is the set of players, $C_i$ is the strategy set of player $i$, and $u_i$ is the utility (or payoff) function of player $i$. \\

$[\cdot]$ &
\noindent Pure strategy in a strategic game. \\

$-j$ &
\noindent The remaining set after removing element $j$ from $N$, i.e., $-j = N\setminus\{j\}$.
Example: if $N=\{1,2,3\}$, then $-2=\{1,3\}$. \\

$\sigma_i$ &
\noindent A randomized strategy of player $i$. Example: if $C_1=\{[a_1],[b_1]\}$, then $\Delta(C_1)=\{p[a_1] + (1-p)[b_1] \mid 0\le p\le1\}$,
and $\sigma_1 = 0.5[a_1] + 0.5[b_1]$. \\

$\sigma_i(c_i)$ &
\noindent The probability assigned to pure strategy $c_i$ by randomized strategy $\sigma_i$. Example: $\sigma_1([a_1]) = 0.5$. \\

$\sigma_S$ &
\noindent The joint randomized strategy profile of all players in subset $S\subset N$. When $S=N$, the subscript can be omitted.
Example: $\sigma_1 = 0.5[a_1] + 0.5[b_1]$, $\sigma_2 = [a_2]$,
then $\sigma_{\{1,2\}} = (\sigma_1, \sigma_2)$. \\
\hline
\end{tabular}

\endgroup
\end{table*}


\section{Infinite Skeleton-data Game}
\label{sec:isg}

Game theory has been widely applied in the field of deep learning in recent years. In 2022, Li et al. \cite{Li2022MDLBP} proposed MDLBP, which models local discrepancies through multi-dimensional geometric relationships, reflecting a game-theoretic concept of “competition and cooperation” within the feature space. Mao et al. \cite{Mao2022CMD} introduced CMD, where a collaborative game mechanism between modalities is established through bidirectional distillation. In 2023, Dong et al. \cite{Dong2023HiCLR} proposed HiCLR with a hierarchical contrastive mechanism, enabling multi-scale features to be progressively optimized through strategic competition. Yu et al. \cite{Yu2024GeoExplainer} proposed GeoExplainer, which integrates helical theory and geometric masking to enhance interpretability in spatiotemporal GCNs, using PMI constraints to generate geometric perturbation masks, offering a new geometric perspective for game-theoretic self-explainable learning. In 2024, Li et al. \cite{Li2024BICAM} introduced BI-CAM, where positive and negative mutual information is modeled to generate interpretable attention maps, essentially reflecting a local game process between cooperation and competition. In the same year, they  \cite{Li2024GTCAM} further proposed GT-CAM, incorporating Shapley value allocation from cooperative game theory to characterize collaborative and competitive interactions among structural nodes in feature interpretability analysis. Similar to these studies, M\textsuperscript{3}GCLR also employs game theory as a key tool to characterize feature interaction and drive optimization. Therefore, its methodological design is fundamentally supported by the mini-max game theory.

The mathematical symbols used in this section are summarized in Table~\ref{tab:symbols}.

Game theory early studies primarily focus on finite games with finite strategy sets \cite{rapoport2012game}, and the extension to infinite strategy games also lacks adaptation to data-driven scenarios, particularly for skeleton data. Therefore, this paper constructs a novel equilibrium framework named the Infinite Skeleton-data Game (ISG). 

To facilitate the theoretical formulation, we first briefly introduce the concept of Mini-Max Game (M\textsuperscript{2}G) and Nash equilibrium.

\begin{definition}[Mini-Max Game, M\textsuperscript{2}G\cite{rapoport2012game}]
\label{def:m2g}
Given a two-person strategic-form game $\Gamma = (\{1,2\}, C_1, C_2, u_1, u_2)$, the mini-max condition is expressed as:
\begin{equation}
\label{eq:1}
u_1 + u_2 = 0.
\end{equation}
Then $\Gamma$ is called a Mini-Max Game (M\textsuperscript{2}G).
\end{definition}
\begin{definition}[(Nash) equilibrium \cite{rapoport2012game}]
\label{def:ne}
Given a game $\Gamma = (N,(C_i)_{i \in N}, (u_i)_{i \in N})$, a randomized strategy profile $\sigma$ is a(n) (Nash) equilibrium iff for all $i \in N$,
\begin{equation}
\label{eq:2}
\sigma_i \in \mathop{\arg\max}_{\tau_i \in \Delta(C_i)}u_i(\tau_i, \sigma_{-i}).
\end{equation}
That is, each $\sigma_i$ must be the best response to $\sigma_{-i}$. Equivalently, if $\sigma$ is an equilibrium of $\Gamma$, it satisfies:
\begin{equation}
\label{eq:3}
\sigma_i(c_i) = 0 \Leftrightarrow c_i \notin \mathop{\arg\max}_{d_i \in C_i}u_i(d_i, \sigma_{-i}), \forall i \in N.
\end{equation}
\end{definition}

We now introduce the definition of ISG:
\begin{definition}[Infinite Skeleton-data Game, ISG]
\label{def:isg}
Consider an infinite strategic-form game $\Gamma = (N,(C_i)_{i \in N}, (u_i(\sigma))_{i \in N})$, $\sigma=(\sigma_i)_{i \in N}$, $\sigma_i \in \Delta(C_i)$. For skeleton-based contrastive learning, denote the encoder parameters vector corresponding to the skeleton data as $\boldsymbol{\uptheta}_i$, and the encoder set as $E = \{E_i\}_{i=1}^{|N|}$, where $|N|$ is the number of elements in the set $N$. Let $E=N$, $\sigma_i=\phi(\boldsymbol{\uptheta}_i)$, and $\boldsymbol{\uptheta} = (\boldsymbol{\uptheta}_i)_{i \in E}$. Assume a continuous mapping $\phi: \boldsymbol{\Theta}_i \mapsto \Delta(C_i)$,  where $\boldsymbol{\Theta}_i \subset \mathbb{R}^{n}$ is the embedding space of $\boldsymbol{\uptheta}_i$, and $n$ is the dimension of $\boldsymbol{\uptheta}_i$. Then the game $\Gamma$ is defined as an Infinite Skeleton-data Game (ISG), denoted by $\Gamma_S = (E,(\boldsymbol{\uptheta}_i)_{i \in E}, (u_i(\boldsymbol{\uptheta})_{i \in E})$.
\end{definition}

Then, to ensure the existence of equilibrium in ISG, based on Definition~\ref{def:ne} and Definition~\ref{def:isg}, we propose the following theorem. It guarantees that when the ISG satisfies certain conditions, the equilibrium must exist. The proof of Theorem~\ref{the:isg} is provided in Appendix~\ref{ap:proof}.

\begin{theorem}[Equilibrium Theorem for ISG]
\label{the:isg}
If the polynomial function of mutual information functions serves as the utility function, and ISG $\Gamma_S = (E,(\boldsymbol{\uptheta}_i)_{i \in E}, (u_i(\boldsymbol{\uptheta}))_{i \in E})$ is defined on a bounded and closed set $\boldsymbol{\Theta}_i$, then the equilibrium of ISG exists.
\end{theorem}

\begin{figure*}
  \begin{center}
  \includegraphics[width=0.85\textwidth]{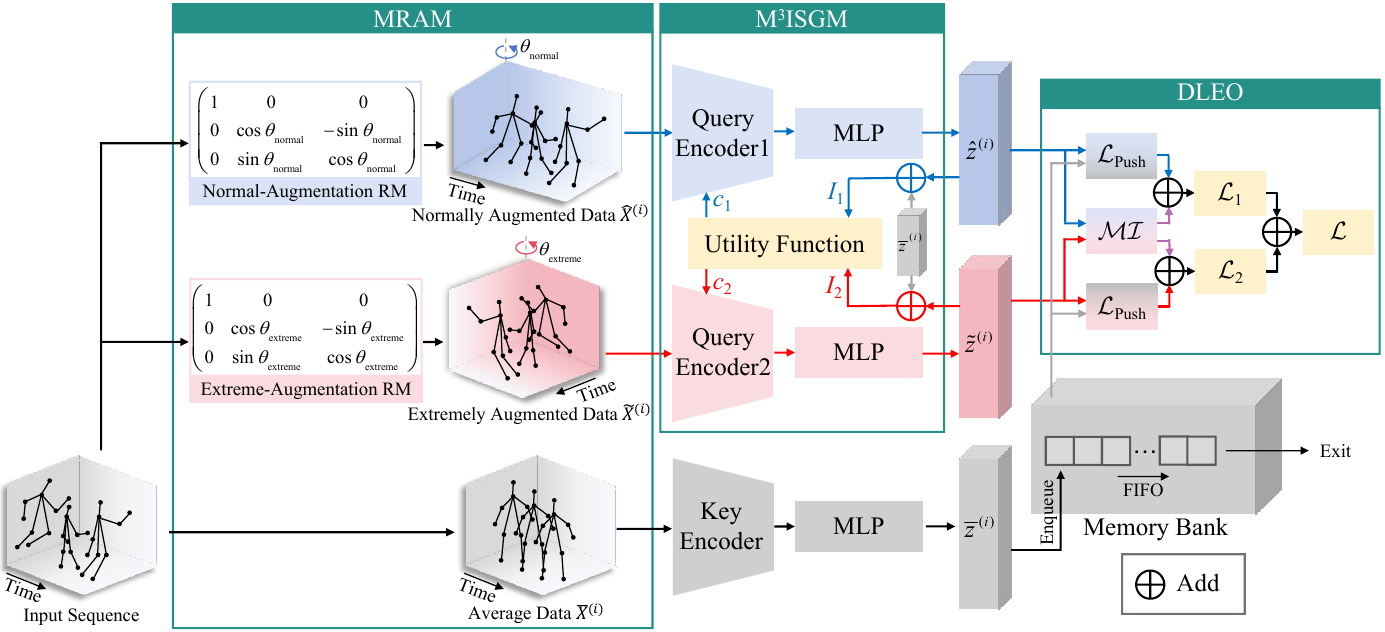}\\
  \caption{Pipeline of the proposed M\textsuperscript{3}GCLR. The input sequence $\mathbf{X}^{(i)}$ is first processed by the Multi-view Rotation-based Augmentation Module, where the normal-augmentation Rotation Matrix, extreme-augmentation Rotation Matrix, and batch averaging are applied to generate a normally augmented data, an extremely augmented data, and the average data, respectively. These three views are fed into query encoder 1, query encoder 2, and the key encoder to obtain the feature embeddings through an MLP projection head. In the Mutual-information-based Mini-Max Infinite Skeleton-data Game Module, the mean mutual information among feature embeddings is computed to construct the utility functions of the ISG. By updating the encoder parameters to maximize the ISG utilities, strong adversarial feature learning is achieved. Finally, in the Dual-Loss-based Equilibrium Optimizer, the optimization is performed using both the $\mathcal{L}_{\mathrm{Push}}$ loss and the KL-divergence-based $\mathcal{MI}$ objective, forming the final loss $\mathcal{L}$. This process further optimizes the model parameters and ensures the convergence of the ISG.}\label{fig:1}
  \end{center}
\end{figure*}

Since solving the equilibrium of a general ISG is a highly complex problem, we further construct a Mini-Max ISG, and use the following Theorem~\ref{the:m2} to derive the equilibrium of ISG.

\begin{theorem}[Equilibrium Theorem for M\textsuperscript{2}G \cite{rapoport2012game}]
\label{the:m2}
For a M\textsuperscript{2}G, a randomized strategy profile $(\sigma_1, \sigma_2)$ is a Nash equilibrium iff:
\begin{equation}
\label{eq:4}
\begin{aligned}
u_1(\sigma_1,\sigma_2) 
&= \mathop{\min}_{\tau_2 \in \Delta(C_2)} \mathop{\max}_{\tau_1 \in \Delta(C_1)} u_1(\tau_1, \tau_2) \\
&= \mathop{\max}_{\tau_1 \in \Delta(C_1)} \mathop{\min}_{\tau_2 \in \Delta(C_2)} u_1(\tau_1, \tau_2).
\end{aligned}
\end{equation}
\end{theorem}

In summary, to obtain a solvable and strongly adversarial ISG model, we formulate M\textsuperscript{3}GCLR as a Mini-Max mutual-information ISG. Specifically, the normal-augmented encoder and extreme-augmented encoder are regarded as the strategic players in the ISG, and mutual-information functions are adopted as the utility functions to drive the game process. In this way, the constructed Mini-Max ISG not only guarantees the existence of ISG equilibrium but also satisfies the equilibrium condition in Eq.~(\ref{eq:4}), thereby further enhancing the performance of our model.

\section{Method}

\subsection{Overview of M\textsuperscript{3}GCLR}

The proposed Multi-view Mini-Max Infinite Skeleton-data Game Contrastive Learning Network (M\textsuperscript{3}GCLR) follows the basic contrastive learning paradigm of MoCov2~\cite{He2019MoCo}. It mainly consists of three components: the Multi-view Rotation-based Augmentation Module (MRAM), the Mutual-information-based Mini-Max Infinite Skeleton-data Game Module (M\textsuperscript{3}ISGM), and the Dual-Loss-based Equilibrium Optimizer (DLEO). The overall architecture is illustrated in Fig.~\ref{fig:1}.

First, let the input skeleton sequence of MRAM be denoted as $\mathbf{X}^{(i)}$, where its $k$-th joint sequence is represented as Eq.~(\ref{eq:5}):
\begin{equation}
\label{eq:5}
\begin{aligned}
\mathbf{X}_k^{(i)} = \{(x_{i1k},y_{i1k},z_{i1k}),(x_{i2k},y_{i2k},z_{i2k}),...,\\
\quad{(x_{iTk},y_{iYk},z_{iTk})\}},
\end{aligned}
\end{equation}
where $x_{itk}$, $y_{itk}$, $z_{itk}$ are the three-dimensional coordinates of the $k$-th joint in the $t$-th frame of the $i$-th skeleton sequence, and $T$ denotes the total number of frames in $\mathbf{X}^{(i)}$. By incorporating a normal rotation angle $\theta_{normal}$ and an extreme rotation angle $\theta_{extreme}$, multi-view augmentation is applied to capture multi-level motion features. This results in the normally augmented data $\hat{\mathbf{X}}^{(i)}$ and the extremely augmented data $\tilde{\mathbf{X}}^{(i)}$. Meanwhile, a temporal average operation is used to generate the average data $\bar{\mathbf{X}}^{(i)}$. These three sequences are encoded by query encoder 1, query encoder 2 and the key encoder, producing the corresponding feature representations $\hat{\mathbf{z}}^{(i)}$, $\tilde{\mathbf{z}}^{(i)}$, and $\bar{\mathbf{z}}^{(i)}$, respectively.

Then, to model the adversarial between the two query encoders, the features $\hat{\mathbf{z}}^{(i)}$, $\tilde{\mathbf{z}}^{(i)}$, and $\bar{\mathbf{z}}^{(i)}$ are fed into M\textsuperscript{3}ISGM for Mini-Max Mutual-information ISG, updating the parameters of both query encoders during each training iteration.

\begin{figure*}
  \begin{center}
  \includegraphics[width=0.5\textwidth]{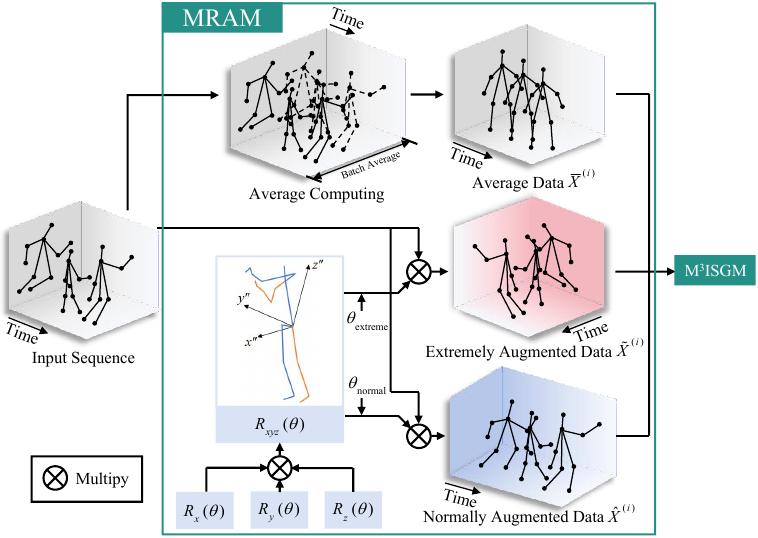}\\
  \caption{Block diagram of MRAM. Rotation Matrices (RMs) around the $x$, $y$, $z$ axes $\mathbf{R}_x(\theta)$, $\mathbf{R}_y(\theta)$, $\mathbf{R}_z(\theta)$ are combined through matrix multiplication to obtain a multi-axis RM $\mathbf{R}_{xyz}(\theta)$. After the input skeleton sequence is processed by the Multi-view Rotation-based Augmentation Module, three transformed views are generated: (a) the average data $\bar{\mathbf{X}}^{(i)}$ obtained by batch averaging, (b) the normally augmented data $\hat{\mathbf{X}}^{(i)}$ produced by multiplying the input with the normal-augmented RM $\mathbf{R}_{xyz}(\theta_{normal})$, and (c) the extremely augmented data $\tilde{\mathbf{X}}^{(i)}$ derived from the extreme-augmented RM $\mathbf{R}_{xyz}(\theta_{extreme})$. These augmented sequences are then fed into the subsequent Mutual-information-based Mini-Max Infinite Skeleton-data Game Module to facilitate robust adversarial representation learning.}\label{fig:2}
  \end{center}
\end{figure*}

Finally, to maximize discriminability while reducing redundant information between normal and extreme augmentations, the mutual information between $\hat{\mathbf{z}}^{(i)}$ and $\tilde{\mathbf{z}}^{(i)}$ is optimized adversarially. The features are fed into DLEO, yielding the normal augmentation loss $\mathcal{L}_1$ and the extreme augmentation loss $\mathcal{L}_2$. Their weighted average forms a dual-loss contrastive learning objective $\mathcal{L}$, which jointly trains the two query encoders in an equilibrium-driven manner, enabling more robust and distinctive skeleton representation learning.

\subsection{Multi-view Rotation-based Augmentation Module}

In data augmentation strategies, rotation is a common spatial transformation that enhances model robustness and generalization by enabling adaptation to different viewpoints and angles \cite{Gao2021ContrastiveSSL}. Therefore, we design a Multi-view Rotation-based Augmentation Module (MRAM), which applies multiple Rotation Matrices (RMs) to the input sequence to generate normally augmented data and extremely augmented data, corresponding to motion detail preservation and motion pattern augmentation, respectively. Meanwhile, an average data is produced by computing the temporal average of the input sequence $\mathbf{X}^{(i)}$. The block diagram of MRAM is illustrated in Fig.~\ref{fig:2}.

First, for the input sequence $\mathbf{X}^{(i)}$ is defined in Eq.~(\ref{eq:5}), the RMs around the $x$, $y$, and $z$ axes are denoted as $\mathbf{R}_x(\theta)$, $\mathbf{R}_y(\theta)$, and $\mathbf{R}_z(\theta)$, respectively, as shown in Eqs.~(\ref{eq:6})--(\ref{eq:8}):
\begin{equation}
\label{eq:6}
\mathbf{R}_x(\theta) =
\begin{pmatrix}
1 & 0 & 0 \\
0 & \cos\theta & -\sin\theta \\
0 & \sin\theta & \cos\theta
\end{pmatrix},
\end{equation}
\begin{equation}
\label{eq:7}
\mathbf{R}_y(\theta) =
\begin{pmatrix}
\cos\theta & 0 & \sin\theta \\
0 & 1 & 0 \\
-\sin\theta & 0 & \cos\theta
\end{pmatrix},
\end{equation}
\begin{equation}
\label{eq:8}
\mathbf{R}_z(\theta) =
\begin{pmatrix}
\cos\theta & -\sin\theta & 0 \\
\sin\theta & \cos\theta & 0 \\
0 & 0 & 1
\end{pmatrix},
\end{equation}
where $\theta$ denotes the rotation angle. By multiplying the above RMs, we obtain the combined Rotation Matrix (RM) $\mathbf{R}_{xyz}(\theta)$, as expressed in Eq.~(\ref{eq:9}):
\begin{equation}
\label{eq:9}
\mathbf{R}_{xyz}(\theta) = \mathbf{R}_x(\theta)\mathbf{R}_y(\theta)\mathbf{R}_z(\theta).
\end{equation}

Next, the input sequence $\mathbf{X}^{(i)}$ is fed into MRAM, where a rotation angle $\theta_{normal}$ is randomly selected from a small-angle range $[-\hat{\theta}, \hat{\theta}]$ to generate a normal-augmentation RM $\mathbf{R}_{xyz}(\theta_{normal})$. By applying $\mathbf{R}_{xyz}(\theta_{normal})$ to $\mathbf{X}^{(i)}$, the normally augmented data $\hat{\mathbf{X}}^{(i)}$ is obtained, as expressed in Eq.~(\ref{eq:10}):
\begin{equation}
\label{eq:10}
\hat{\mathbf{X}}^{(i)} = \mathbf{R}_{xyz}(\theta_{normal}) \cdot \mathbf{X}^{(i)}.
\end{equation}
Here, $\hat{\mathbf{X}}^{(i)}$ denotes the normally augmented data, $\mathbf{X}^{(i)}$ is the input sequence, and $\mathbf{R}_{xyz}(\theta_{normal})$ is the normal-augmentation RM.

\begin{figure}[t]
  \begin{center}
  \includegraphics[width=\columnwidth]{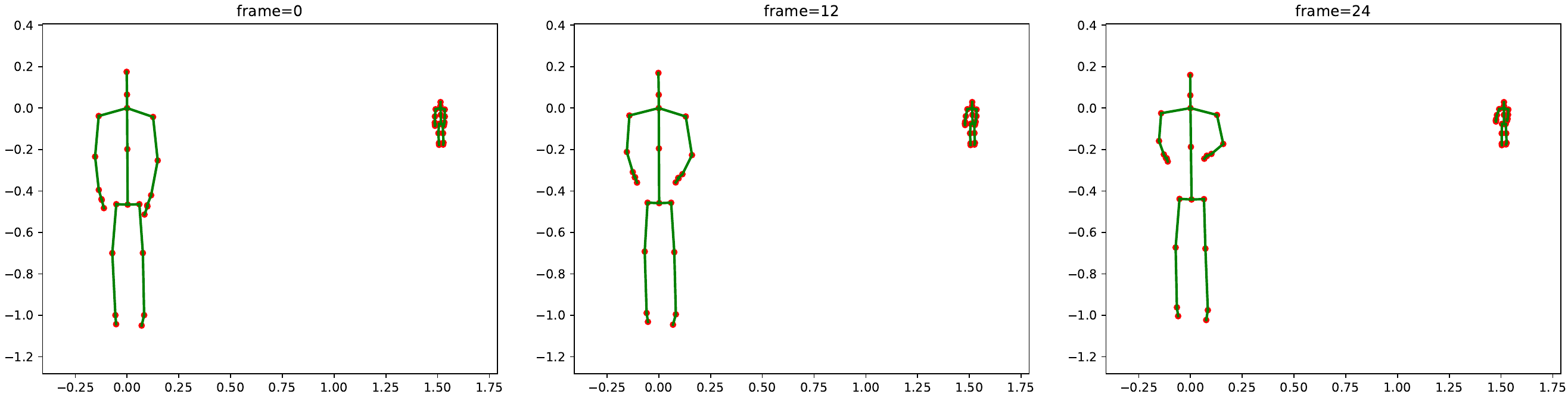}\\
  \caption{Visualizations of mean-motion sequences (the another isolated sequence represents another skeleton instance).}\label{fig:3}
  \end{center}
\end{figure}

\begin{figure*}[t]
  \begin{center}
  \includegraphics[width=0.65\textwidth]{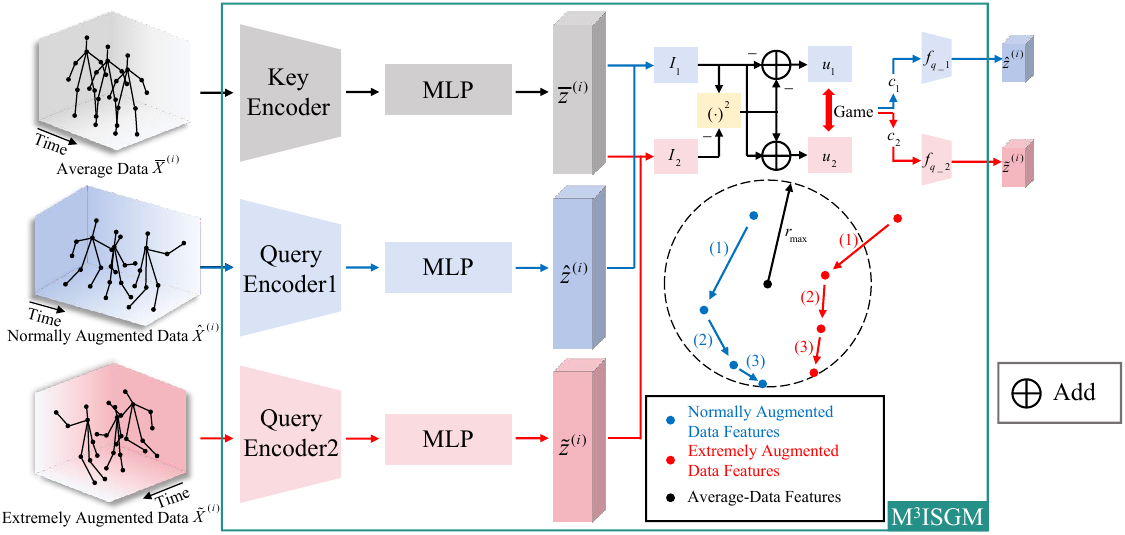}\\
  \caption{Block diagram of M\textsuperscript{3}ISGM. The average data $\bar{\mathbf{X}}^{(i)}$, the normally augmented data $\hat{\mathbf{X}}^{(i)}$, and the extremely augmented data $\tilde{\mathbf{X}}^{(i)}$ produced by the Multi-view Rotation-based Augmentation Module are fed into the Mutual-information-based Mini-Max Infinite Skeleton-data Game Module (M\textsuperscript{3}ISGM). After passing through the key encoder, query encoder 1 (encoder 1) and query encoder 2 (encoder 2), their respective feature representations $\bar{\mathbf{z}}^{(i)}$, $\hat{\mathbf{z}}^{(i)}$, and $\tilde{\mathbf{z}}^{(i)}$ are obtained. The mutual information between $\hat{\mathbf{z}}^{(i)}$ and $\bar{\mathbf{z}}^{(i)}$ is computed as $I_1$, and the mutual information between $\tilde{\mathbf{z}}^{(i)}$ and $\bar{\mathbf{z}}^{(i)}$ is computed as $I_2$. The squared difference between $I_1$ and $I_2$ is added to $I_1$ is obtained encoder 2's utility function $u_2$, while encoder 1's utility function is the negation of encoder 2's, i.e. $u_1 = -u_2$. Based on the obtained equilibrium solution, the parameters of query encoder 1 and query encoder 2 (i.e., $f_{q\_1}$ and $f_{q\_2}$) are updated accordingly, which further provides support for the equilibrium optimization based on dual-loss learning. }\label{fig:4}
  \end{center}
\end{figure*}

Meanwhile, $\mathbf{X}^{(i)}$ is also input to MRAM, where a rotation angle $\theta_{extreme}$ is randomly selected from a large-angle range $[-\tilde{\theta}, \tilde{\theta}]$ to generate the extreme-augmentation RM $\mathbf{R}_{xyz}(\theta_{extreme})$. Applying $\mathbf{R}_{xyz}(\theta_{extreme})$ to $\mathbf{X}^{(i)}$ yields the extremely augmented data $\tilde{\mathbf{X}}^{(i)}$, as given in Eq.~(\ref{eq:11}):
\begin{equation}
\label{eq:11}
\tilde{\mathbf{X}}^{(i)} = \mathbf{R}_{xyz}(\theta_{extreme}) \cdot \mathbf{X}^{(i)}.
\end{equation}
Here, $\tilde{\mathbf{X}}^{(i)}$ denotes the extremely augmented data, $\mathbf{X}^{(i)}$ is the input sequence, and $\mathbf{R}_{xyz}(\theta_{extreme})$ is the extreme-augmentation RM.

Finally, since most frames in a complete skeleton sequence exhibit no significant movement, as shown in Fig.~\ref{fig:3}, we observe that the positions of most joints remain stable without notable variations. Based on this observation, the average data of $\mathbf{X}^{(i)}$ is computed using Eq.~(\ref{eq:12}):
\begin{equation}
\label{eq:12}
\bar{\mathbf{X}}^{(i)} = \frac{1}{B} \sum_{b=1}^{B}\mathbf{X}_b^{(i)},
\end{equation}
where $\bar{\mathbf{X}}^{(i)}$ represents the averaged sequence of the input sequence $\mathbf{X}^{(i)}$, and $\mathbf{X}_b^{(i)}$ denotes the $b$-th skeleton sequence in batch $B$ at the $i$-th iteration.

\subsection{Mutual-information-based Mini-Max Infinite Skeleton-data Game Module}
\label{sec:m3isgm}

Different observation viewpoints can lead to significant variations in the 3D coordinate representation and spatio-temporal motion patterns of the same action. However, existing contrastive learning methods lack effective adversarial modeling, which fundamentally limits further performance improvement. Therefore, based on the theoretical extension of game theory introduced in Section~\ref{sec:isg}, we propose the Mutual-information-based Mini-Max Infinite Skeleton-data Game Module (M\textsuperscript{3}ISGM). By formulating a game-theoretic optimization, M\textsuperscript{3}ISGM maximizes the discrepancy between augmented data $\hat{\mathbf{X}}^{(i)}$ and $\tilde{\mathbf{X}}^{(i)}$ generated by MRAM and the average data $\bar{\mathbf{X}}^{(i)}$, thereby enhancing the adversarial capability of the network. The block diagram of M\textsuperscript{3}ISGM is shown in Fig.~\ref{fig:4}.

First, M\textsuperscript{3}ISGM receives the augmented data $\hat{\mathbf{X}}^{(i)}$ and $\tilde{\mathbf{X}}^{(i)}$ from MRAM, along with the average data $\bar{\mathbf{X}}^{(i)}$. All of them are fed into the encoders of the backbone network ST-GCN (including the key encoder $f_k(\cdot)$ and the query encoders $f_{q\_1}(\cdot)$, $f_{q\_2}(\cdot)$) and the MLP projection head $g(\cdot)$. The corresponding feature representations are obtained as: $\hat{\mathbf{z}}^{(i)} = g(f_{q\_1}(\hat{\mathbf{X}}^{(i)}))$, $\tilde{\mathbf{z}}^{(i)} = g(f_{q\_2}(\tilde{\mathbf{X}}^{(i)}))$, $\bar{\mathbf{z}}^{(i)} = g(f_k(\bar{\mathbf{X}}^{(i)}))$.

\begin{figure*}[t]
  \begin{center}
  \includegraphics[width=0.75\textwidth]{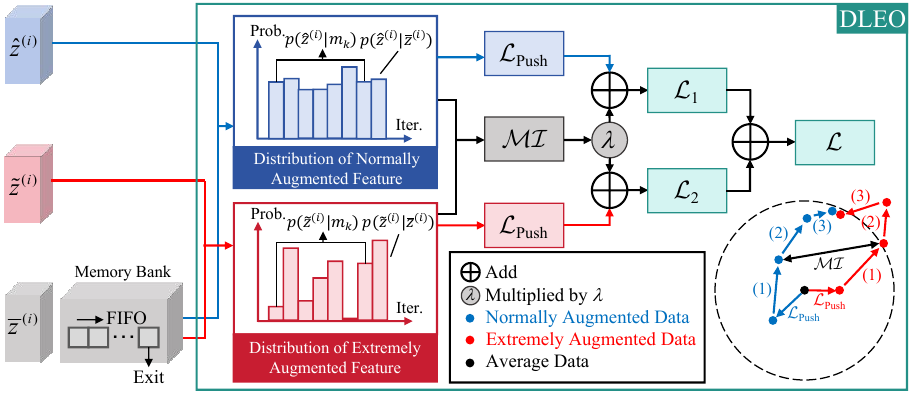}\\
  \caption{Block diagram of DLEO. After obtaining encoded features from the Mutual-information-based Mini-Max Infinite Skeleton-data Game Module, we feed them into the Dual-Loss-based Equilibrium Optimizer (DLEO). In DLEO, we first compute the distributions among the negative feature $\mathbf{m}_k$, the average feature $\bar{\mathbf{z}}^{(i)}$, the normally augmented feature $\hat{\mathbf{z}}^{(i)}$, and the extremely augmented feature $\tilde{\mathbf{z}}^{(i)}$, formulated as $p(\hat{\mathbf{z}}^{(i)}|\bar{\mathbf{z}}^{(i)})$, $p(\tilde{\mathbf{z}}^{(i)}|\bar{\mathbf{z}}^{(i)})$, $p(\hat{\mathbf{z}}^{(i)}|\mathbf{m}_k)$, and $p(\tilde{\mathbf{z}}^{(i)}|\mathbf{m}_k)$. Based on these distributions, the InfoNCE loss $\mathcal{L}_{\mathrm{Push}}$ and the KL divergence $\mathcal{MI}$ are calculated. Finally, the overall loss function $\mathcal{L}$ is computed from $\mathcal{L}_{\mathrm{Push}}$ and $\mathcal{MI}$, enabling multi-view Mini-Max game-driven contrastive learning to achieve more robust representation learning for skeleton-based action recognition. }\label{fig:5}
  \end{center}
\end{figure*}

Next, a Mini-Max Infinite Skeleton-data Game (ISG) is formulated. Let $\Gamma_S=(\{1,2\}, \boldsymbol{\uptheta}_1, \boldsymbol{\uptheta}_2, u_1, u_2)$, where $E = \{1,2\}$, query encoder 1 and query encoder 2 correspond to encoder 1 and encoder 2 in $\Gamma_S$, respectively, and encoder 1 is parameterized by $\boldsymbol{\uptheta}_1$, encoder 2 is parameterized by $\boldsymbol{\uptheta}_2$, and $u_1$ and $u_2$ denote the utility functions of encoder 1 and 2, respectively. To strengthen the adversarial modeling and utilize the polynomial function of mutual information as the utility functions, we define:
\begin{equation}
\label{eq:13}
u_1(\boldsymbol{\uptheta}) = -I_1 - (I_1 - I_2)^2, \quad{\boldsymbol{\uptheta} = (\boldsymbol{\uptheta}_1, \boldsymbol{\uptheta}_2)},
\end{equation}
\begin{equation}
\label{eq:14}
u_2(\boldsymbol{\uptheta}) = I_1 + (I_1 - I_2)^2,
\end{equation}
where $I_1 = I(\hat{\mathbf{z}}^{(i)}; \bar{\mathbf{z}}^{(i)})$ and $I_2 = I(\tilde{\mathbf{z}}^{(i)}; \bar{\mathbf{z}}^{(i)})$, and $I(\cdot; \cdot)$ denotes the mutual information of a pair of feature samples, which are negatively correlated to the statistical correlation between the normally and extremely augmented data. Mini-Max ISG $\Gamma_S$ enables adversarial optimization that maximizes the discrepancy between $\tilde{\mathbf{z}}^{(i)}$ and $\bar{\mathbf{z}}^{(i)}$, while matching the discrepancy between $\hat{\mathbf{z}}^{(i)}$ and $\bar{\mathbf{z}}^{(i)}$.

However, the adversarial learning under Mini-Max ISG alone may not fully guarantee the maximization of motion-informative features while minimizing redundant components. Therefore, the resulting features $\hat{\mathbf{z}}^{(i)}$, $\tilde{\mathbf{z}}^{(i)}$, and $\bar{\mathbf{z}}^{(i)}$ are further fed into the DLEO module to ensure the final convergence of the game.

\subsection{Dual-Loss-based Equilibrium Optimizer}

To maximize action-specific information across different sample instances while minimizing redundant information between the normally augmented and extremely augmented features, M\textsuperscript{3}ISGM models the normal-augmentation query encoder and the extreme-augmentation query encoder as two encoders in an ISG, using the average data as the anchor reference to construct a Mini-Max mutual-information ISG and obtain its equilibria. However, not all equilibria inherently satisfy the optimization goals of the network. Therefore, introducing additional constraints to search for a desirable equilibrium becomes essential for the convergence of the Mini-Max mutual-information ISG.

Inspired by focal effect theory in game theory \cite{rapoport2012game}, and motivated by the established equivalence between Dual-Loss-based Equilibrium Optimizer (DLEO) and the Mini-Max Infinite Skeleton-data Game $\Gamma_S$ introduced in the Section~\ref{sec:m3isgm}, we design DLEO to filter the equilibria. Specifically, the InfoNCE loss \cite{He2019MoCo} between $\bar{\mathbf{z}}^{(i)}$ and $\hat{\mathbf{z}}^{(i)}$, as well as $\bar{\mathbf{z}}^{(i)}$, $\tilde{\mathbf{z}}^{(i)}$, is computed to maximize action-specific information relative to the average feature. Meanwhile, the KL divergence between $\hat{\mathbf{z}}^{(i)}$, $\tilde{\mathbf{z}}^{(i)}$ is computed to minimize redundant motion representations across augmented views. In this way, the ISG converges toward a solution that simultaneously maximizes informative motion cues and suppresses redundant information. The block diagram of DLEO is shown in Fig.~\ref{fig:5}, and the proof of equivalence between M\textsuperscript{3}ISGM and DLEO is provided Appendix~\ref{ap:proof}.

In a self-supervised setting, the objective of encoder 1 (the normal-augmentation query encoder) is to maximize learning on the detailed features of the input sequence while minimizing redundant information shared with the extremely augmented data. Therefore, we formulate the loss function of encoder 1 as in Eq.~(\ref{eq:15}):
\begin{equation}
\label{eq:15}
\mathcal{L}_1 = \mathcal{L}_{\mathrm{Push}}(\hat{\mathbf{z}}^{(i)}, \bar{\mathbf{z}}^{(i)}) + \lambda \cdot \mathcal{MI}(\hat{\mathbf{z}}^{(i)}; \tilde{\mathbf{z}}^{(i)}),
\end{equation}
where $\mathcal{MI}(\hat{\mathbf{z}}^{(i)}; \tilde{\mathbf{z}}^{(i)})$ measures the discrepancy between the normally augmented representation $\hat{\mathbf{z}}^{(i)}$ and the extremely augmented representation $\tilde{\mathbf{z}}^{(i)}$ using KL divergence, $\lambda \ge 0$ is a redundancy-penalizing coefficient, and $\mathcal{L}_{\mathrm{Push}}(\cdot)$ is the contrastive InfoNCE loss defined as:
\begin{equation}
\label{eq:16}
\begin{aligned}
\mathcal{L}_{\mathrm{Push}}(\hat{\mathbf{z}}^{(i)}, \bar{\mathbf{z}}^{(i)}) &= \\
&\hspace{-3em}{-\mathrm{log}\frac{\mathrm{exp}(\hat{\mathbf{z}}^{(i)} \cdot \bar{\mathbf{z}}^{(i)} / \tau)}{\mathrm{exp}(\hat{\mathbf{z}}^{(i)} \cdot \bar{\mathbf{z}}^{(i)} / \tau) + \sum_{k=1}^{|\mathbf{M}|}\mathrm{exp}(\hat{\mathbf{z}}^{(i)} \cdot \mathbf{m}_k / \tau)}},
\end{aligned}
\end{equation}
where $\bar{\mathbf{z}}^{(i)}$ is the averaged representation of input sequence $\mathbf{X}^{(i)}$, $\mathbf{m}_k$ is the $k$-th negative sample from the memory bank $\mathbf{M}$, $\tau$ is the temperature hyperparameter, and $|\mathbf{M}|$ is the memory bank size.

Thus, for encoder 1, finding the equilibrium of the game reduces to solving the unconstrained optimization problem:
\begin{equation}
\label{eq:17}
\mathop{\min}_{\hat{\mathbf{z}}^{(i)}}\mathcal{L}_1.
\end{equation}

Similarly, encoder 2 (the extreme-augmentation query encoder) is designed to maximize learning on the global features of the skeleton sequence while minimizing redundant information shared with the normal-augmented data. Therefore, we define the loss function of encoder 2 as in Eq.~(\ref{eq:18}):
\begin{equation}
\label{eq:18}
\mathcal{L}_2 = \mathcal{L}_{\mathrm{Push}}(\tilde{\mathbf{z}}^{(i)}, \bar{\mathbf{z}}^{(i)}) + \lambda \cdot \mathcal{MI}(\tilde{\mathbf{z}}^{(i)}; \hat{\mathbf{z}}^{(i)}).
\end{equation}
Likewise, the corresponding optimization problem is:
\begin{equation}
\label{eq:19}
\mathop{\min}_{\tilde{\mathbf{z}}^{(i)}}\mathcal{L}_2.
\end{equation}
All physical quantities in Eqs.~{(\ref{eq:18})--(\ref{eq:19})} share identical meanings as described previously in Eqs.~{(\ref{eq:15})--(\ref{eq:17})}.

As mentioned above, the KL divergence is employed in this work to measure the discrepancy between the normal-augmented features and the extreme-augmented features. However, since the final extracted feature representations are still of high dimensionality, it is difficult to directly estimate their accurate distributions. Therefore, inspired by Wang and Qi \cite{wang2022contrastive}, we utilize the positive features output by the key encoder together with a large number of negative features stored in the memory bank $\mathbf{M}$ to approximate the conditional distributions of high-dimensional features. Specifically, the conditional distributions of the features are formulated as shown in Eqs.~(\ref{eq:20})--(\ref{eq:21}):
\begin{equation}
\label{eq:20}
p(\hat{\mathbf{z}}^{(i)}|\bar{\mathbf{z}}^{(i)}) = \frac{\mathrm{exp}(\hat{\mathbf{z}}^{(i)} \cdot \bar{\mathbf{z}}^{(i)} / \tau)}{\mathrm{exp}(\hat{\mathbf{z}}^{(i)} \cdot \bar{\mathbf{z}}^{(i)} / \tau) + \sum_{k=1}^{|\mathbf{M}|}\mathrm{exp}(\hat{\mathbf{z}}^{(i)} \cdot \mathbf{m}_k / \tau)},
\end{equation}
\begin{equation}
\label{eq:21}
p(\tilde{\mathbf{z}}^{(i)}|\bar{\mathbf{z}}^{(i)}) = \frac{\mathrm{exp}(\tilde{\mathbf{z}}^{(i)} \cdot \bar{\mathbf{z}}^{(i)} / \tau)}{\mathrm{exp}(\tilde{\mathbf{z}}^{(i)} \cdot \bar{\mathbf{z}}^{(i)} / \tau) + \sum_{k=1}^{|\mathbf{M}|}\mathrm{exp}(\tilde{\mathbf{z}}^{(i)} \cdot \mathbf{m}_k / \tau)},    
\end{equation}
where $|\mathbf{M}|$ denotes the size of the memory bank, $\hat{\mathbf{z}}^{(i)}$, $\tilde{\mathbf{z}}^{(i)}$, and $\bar{\mathbf{z}}^{(i)}$ represent the features extracted from the normally augmented data $\hat{\mathbf{X}}^{(i)}$, the extremely augmented data $\tilde{\mathbf{X}}^{(i)}$ and the average data $\bar{\mathbf{X}}^{(i)}$, respectively. $\tau$ is a temperature hyper-parameter. These distributions characterize the similarity between the current sample and the negative samples stored in the memory bank.

Thus, using the two conditional distributions defined above, the mutual information terms $\mathcal{MI}(\hat{\mathbf{z}}^{(i)}; \tilde{\mathbf{z}}^{(i)})$ and $\mathcal{MI}(\tilde{\mathbf{z}}^{(i)}; \hat{\mathbf{z}}^{(i)})$ can be computed as follows:
\begin{equation}
\label{eq:22}
\mathcal{MI}(\hat{\mathbf{z}}^{(i)}; \tilde{\mathbf{z}}^{(i)}) = 
D_{KL}(p(\hat{\mathbf{z}}^{(i)}|\bar{\mathbf{z}}^{(i)}) 
\lVert p(\tilde{\mathbf{z}}^{(i)}|\bar{\mathbf{z}}^{(i)})),
\end{equation}
\begin{equation}
\label{eq:23}
\mathcal{MI}(\tilde{\mathbf{z}}^{(i)}; \hat{\mathbf{z}}^{(i)}) = 
D_{KL}(p(\tilde{\mathbf{z}}^{(i)}|\bar{\mathbf{z}}^{(i)}) 
\lVert p(\hat{\mathbf{z}}^{(i)}|\bar{\mathbf{z}}^{(i)})),
\end{equation}
where $D_{KL}(\cdot \lVert \cdot)$ denotes the KL divergence between two sample distributions. Since both types of augmented data are used to train the model, the final loss must jointly consider $\mathcal{L}_1$ and $\mathcal{L}_2$. Therefore, the overall loss used in the proposed M\textsuperscript{3}GCLR network is defined as the arithmetic mean of $\mathcal{L}_1$ and $\mathcal{L}_2$, i.e.,
\begin{equation}
\label{eq:24}
\mathcal{L} = \frac{1}{2}(\mathcal{L}_1 + \mathcal{L}_2).
\end{equation}

DLEO explicitly drives the M\textsuperscript{2}G ($\Gamma$) to converge toward an equilibrium that maximizes the discrepancy between the detailed/global features and the average feature, while minimizing the redundant information between the normally and extremely augmented features. Built upon the convergence stability guaranteed by M\textsuperscript{2}G, DLEO enables more robust and discriminative feature learning.

\subsection{Pseudocode of the Algorithm}

The pseudocode of the proposed \textit{M\textsuperscript{3}GCLR network} is shown in Algorithm~\ref{alg:m3gclr}.

\begin{algorithm}[t]
\caption{M\textsuperscript{3}GCLR algorithm.}
\label{alg:m3gclr}
\begin{algorithmic}[1]

\State \textbf{Input:} 
\Statex \makebox[7.7em][l]{$\mathbf{X}$}          \texttt{\# input sequence}
\Statex \makebox[7.7em][l]{$\theta_{normal} \in [-\hat{\theta}, \hat{\theta}]$}  \texttt{\# normal rotation angle}
\Statex \makebox[7.7em][l]{$\theta_{extreme} \in [-\tilde{\theta}, \tilde{\theta}]$}  \texttt{\# extreme rotation angle}
\Statex \makebox[7.7em][l]{$B$}          \texttt{\# batch size}
\Statex \makebox[7.7em][l]{$E$}          \texttt{\# training iteration}
\Statex \makebox[7.7em][l]{$m$}          \texttt{\# momentum parameter}
\Statex \makebox[7.7em][l]{$\tau$}          \texttt{\# temperature}
\State \textbf{Onput:} 
\Statex \makebox[7.7em][l]{$\boldsymbol\uptheta$}          \texttt{\# weights after training}

\vspace{0.5em}
\State \textbf{Procedure:}
\State Initialize encoders and projection heads; copy parameters to momentum teacher branch
\Statex \hspace{0.0em}\texttt{\# initialization}
\State Initialize memory bank $Q \leftarrow \varnothing$
\Statex \hspace{0.0em}\texttt{\# initialize memory bank}

\For{each $i$ in $[1,E]$}
    \State Sample a batch of size $B$ from $\mathbf{X}$, denoted as $\mathbf{X}^{(i)}$.

    \Statex \hspace{1.4em}\texttt{\# =============MRAM==============}
    
    \State Generate normally augmented data:
    \Statex {%
    \centering
    $\hat{\mathbf{X}}^{(i)} = \mathbf{R}_{xyz}(\theta_{normal}) \cdot \mathbf{X}^{(i)}$
    \par
    }

    \State Generate extremely augmented data: \Statex {%
    \centering
    $\tilde{\mathbf{X}}^{(i)} = \mathbf{R}_{xyz}(\theta_{extreme}) \cdot \mathbf{X}^{(i)}$
    \par
    }

    \State Compute average data: $\bar{\mathbf{X}}^{(i)} = \frac{1}{B} \sum_{b=1}^{B} \mathbf{X}_b^{(i)}$

    \Statex \hspace{1.4em}\texttt{\# ==========M\textsuperscript{3}ISGM+DLEO==========}
    
    \State Normally augmented features: $f_{q\_1}(\hat{\mathbf{X}}^{(i)})$

    \State Extremely augmented features: $f_{q\_2}(\tilde{\mathbf{X}}^{(i)})$

    \State Average data features: $f_{k}(\bar{\mathbf{X}}^{(i)})$

    \State Normally augmented projections:
    \Statex {%
    \centering
    $\hat{\mathbf{z}}^{(i)} = g(f_{q\_1}(\hat{\mathbf{X}}^{(i)}))$
    \par}
    
    \State Extremely augmented projections: 
    \Statex {%
    \centering
    $\tilde{\mathbf{z}}^{(i)} = g(f_{q\_2}(\tilde{\mathbf{X}}^{(i)}))$
    \par}

    \State Average data projections: $\bar{\mathbf{z}}^{(i)} = g(f_k(\bar{\mathbf{X}}^{(i)}))$
    
    \State Compute Push loss: $\mathcal{L}_{\mathrm{Push}}(\hat{\mathbf{z}}^{(i)}, \bar{\mathbf{z}}^{(i)})$, $\mathcal{L}_{\mathrm{Push}}(\tilde{\mathbf{z}}^{(i)}, \bar{\mathbf{z}}^{(i)})$

    \State Compute KL divergence:
    \Statex {%
    \centering
    $\mathcal{MI}(\hat{\mathbf{z}}^{(i)}, \tilde{\mathbf{z}}^{(i)})$, $\mathcal{MI}(\tilde{\mathbf{z}}^{(i)}, \hat{\mathbf{z}}^{(i)})$
    \par}
    
    \State Compute overall loss: $\mathcal{L} = \frac{1}{2}(\mathcal{L}_1 + \mathcal{L}_2)$

    \State Update encoders parameters
\EndFor

\vspace{0.3em}

\end{algorithmic}
\end{algorithm}

\section{Experiments}
\subsection{Datasets}
This paper evaluates the feasibility and effectiveness of the proposed method M\textsuperscript{3}GCLR on two public human action datasets, NTU RGB+D~\cite{Amir2016NTU60, liu2019ntu120} and PKU-MMD~\cite{Liu2017PKUMMD}.

\subsubsection{NTU RGB+D} 
The NTU RGB+D dataset, released by Nanyang Technological University (NTU), is a large-scale benchmark for human action recognition using RGB+D data. For NTU RGB+D 60~\cite{Amir2016NTU60}, it contains 60 action classes and 56,578 sequences recorded by three Microsoft Kinect v2 cameras, including RGB videos, depth maps, skeleton data, and infrared sequences. Data were collected from 40 subjects under three camera views. Two protocols are provided: Cross-Subject (X-Sub), where subjects 1–20 are used for training and 21–40 for testing, and Cross-View (X-View), where two views are used for training and one for testing. An extended version, NTU RGB+D 120~\cite{liu2019ntu120}, was released in 2019, increasing the number of classes to 120 and the total samples to about 114,480. It involves 106 subjects and 32 camera setups, offering greater scene diversity and complexity. Similar to NTU60, it adopts Cross-Subject (X-Sub) and Cross-Setup (X-Set) protocols, making it one of the most challenging skeleton-based benchmarks.

\subsubsection{PKU-MMD} 
The PKU-MMD dataset~\cite{Liu2017PKUMMD}, released by the Wangxuan Institute of Peking University in 2017, provides a large-scale multimodal benchmark for human action recognition. It consists of two subsets, Part I and Part II. Part I includes 66 subjects performing 51 actions, producing 20,734 samples and 1,074 independent action instances. Under the C-view protocol, data from the middle and right cameras are used for training, and those from the left camera for testing. Part II contains more complex and dynamic motion scenes with stronger inter-person interactions and occlusions. It follows the same C-view protocol but focuses on large-scale movements and real-world complexity, making it a valuable dataset for evaluating model robustness and generalization in challenging environments.

\subsection{Experimental Settings}
The detailed system configuration of the proposed method is summarized in Table.~\ref{tab:config}.

\begin{table}[t]
\centering
\caption{System Configuration}
\label{tab:config}
\renewcommand{\arraystretch}{1.2}
\begin{tabular}{|c c|}
\hline
System Configuration & Parameters \\ \hline
CUDA Version & Cuda release V11.8.89 \\
Pytorch Version & Pytorch 1.13.0+cu117 \\
Python Version & 3.9.12 \\
GPU Model & Nvidia GeForce RTX3090 (24G) \\ \hline
\end{tabular}
\end{table}

\begin{figure}[t]
  \begin{center}
  \includegraphics[width=\columnwidth]{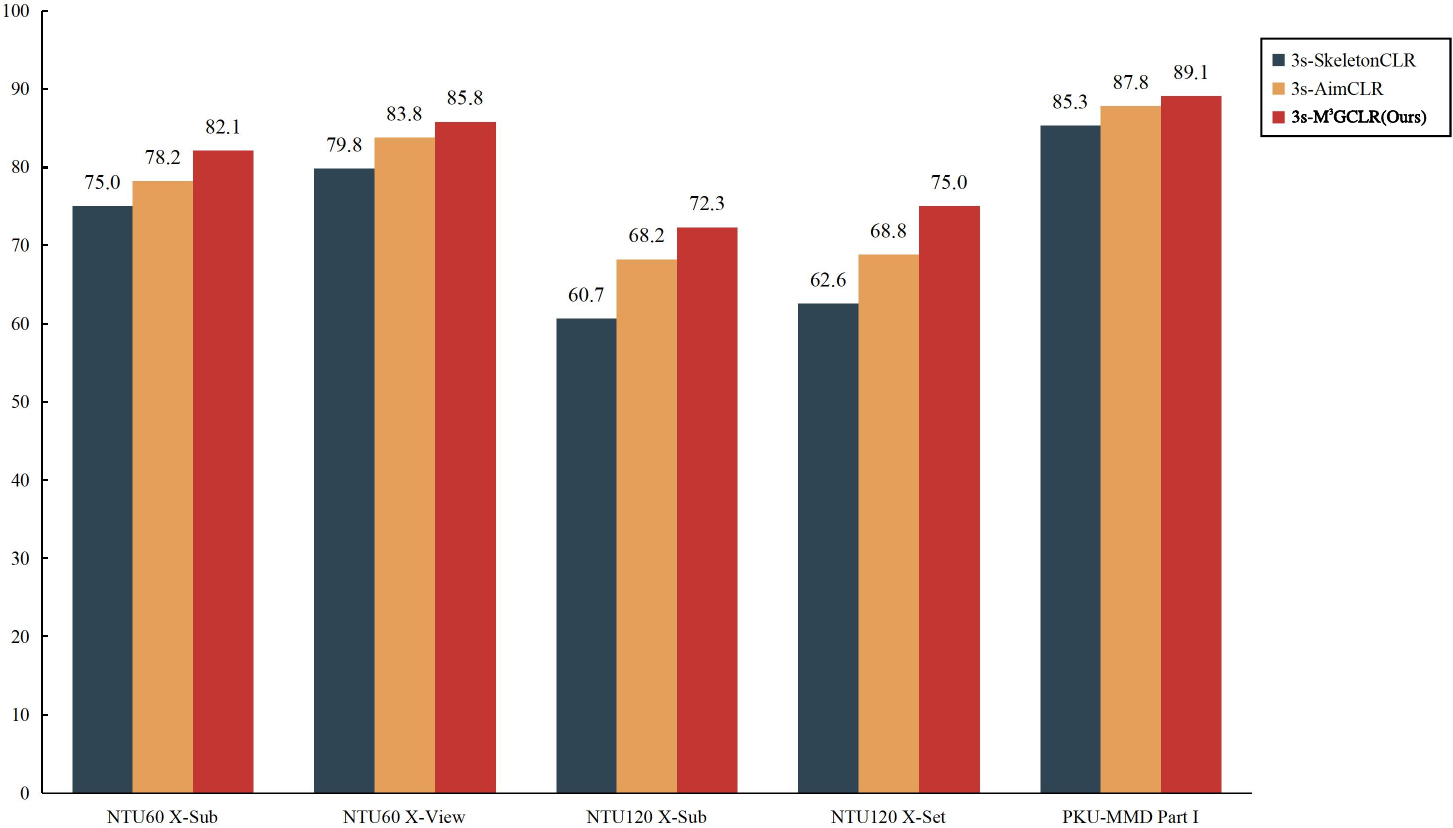}\\
  \caption{Comparison results of the 3s-baseline methods and 3s-M\textsuperscript{3}GCLR on NTU RGB+D 60, NTU RGB+D 120, and PKU-MMD Part I datasets under the linear evaluation protocol. 3s-M\textsuperscript{3}GCLR achieves consistently superior performance across all methods.}\label{fig:8}
  \end{center}
\end{figure}

\begin{figure}[t]
  \begin{center}
  \includegraphics[width=\columnwidth]{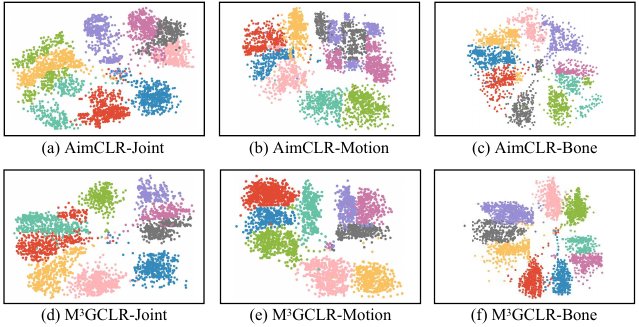}\\
  \caption{Feature distribution visualization of nine representative actions on the NTU RGB+D 60 dataset. The scatter maps show the encoded features after linear evaluation for AimCLR and M\textsuperscript{3}GCLR, where M\textsuperscript{3}GCLR exhibits more compact intra-class clustering and clearer inter-class separation.}\label{fig:6}
  \end{center}
\end{figure}

\begin{figure*}[!t]
  \begin{center}
  \includegraphics[width=\textwidth]{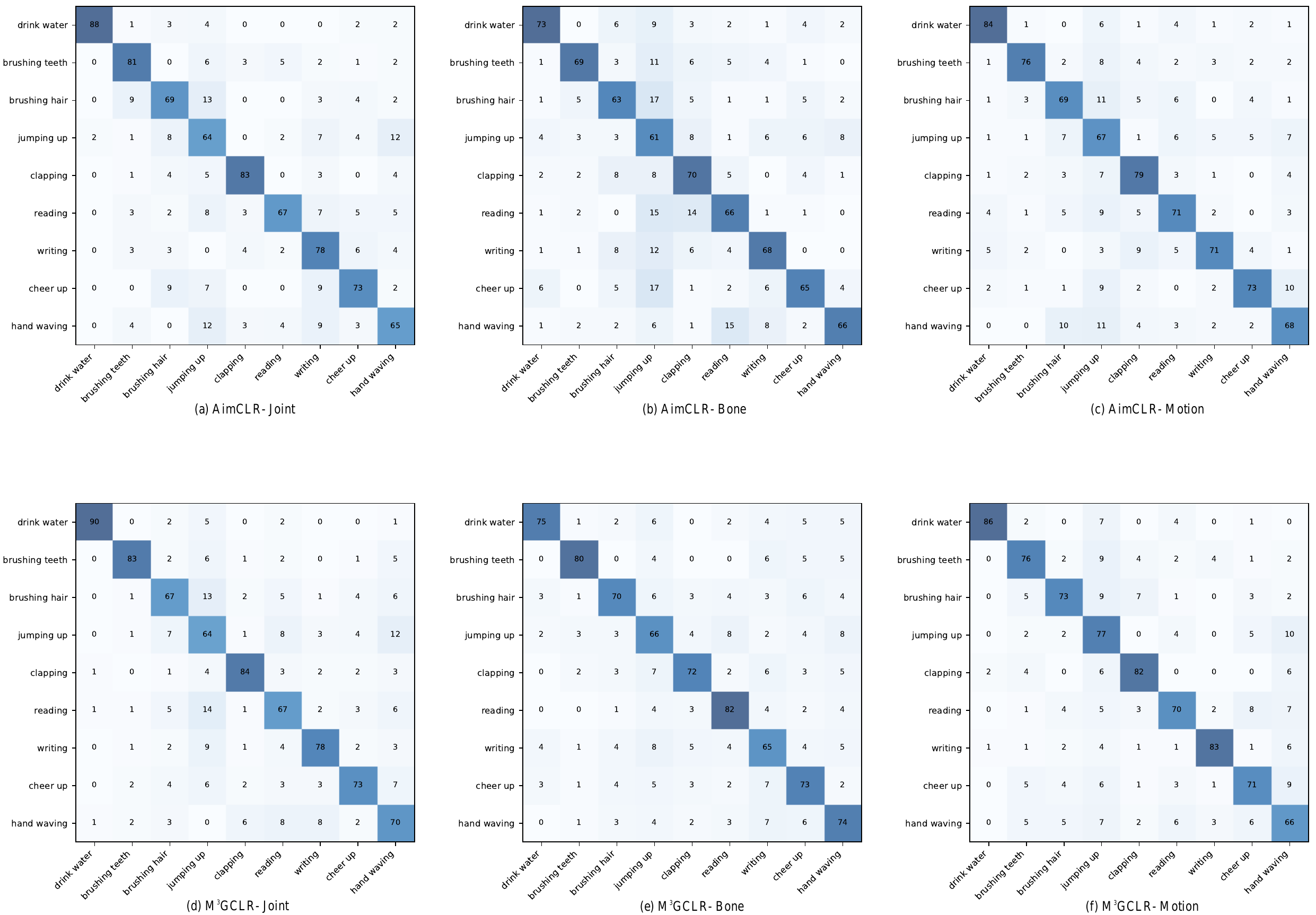}\\
  \caption{Confusion matrices of AimCLR and M\textsuperscript{3}GCLR under the linear evaluation protocol on the NTU RGB+D 60 dataset, visualized for the Joint, Motion, and Bone streams. Each matrix is computed from 100 samples per class across 9 actions. M\textsuperscript{3}GCLR consistently demonstrates higher recognition accuracy and reduced misclassification.}\label{fig:7}
  \end{center}
\end{figure*}

\begin{table*}[!t]
\centering
\caption{Comparison Results of the Proposed M\textsuperscript{3}GCLR and Baseline Methods on NTU RGB+D 60, NTU RGB+D 120, and PKU-MMD Datasets under the Linear Evaluation Protocol}
\label{tab:base}
\renewcommand{\arraystretch}{1.15}
\begin{tabular}{|ccccccc|}
\hline
\multirow{2}{*}{Method} & 
\multirow{2}{*}{Stream} & 
\multicolumn{2}{c}{NTU RGB+D 60 (\%)} & 
\multicolumn{2}{c}{NTU RGB+D 120 (\%)} & 
\multicolumn{1}{c|}{PKU-MMD (\%)} \\ \cline{3-7}
 & & X-Sub Acc.$\uparrow$ & X-View Acc.$\uparrow$ & X-Sub Acc.$\uparrow$ & X-Set Acc.$\uparrow$ & Part I Acc.$\uparrow$ \\ \hline

SkeletonCLR~\cite{Li2021CrossViewConsistency} & Joint & 68.3 & 76.4 & 56.8 & 55.9 & 80.9 \\
AimCLR~\cite{Guo2022AimCLR} & Joint & \underline{74.3} & \underline{79.6} & \underline{63.4} & \underline{63.4} & \underline{83.0} \\
\textbf{M\textsuperscript{3}GCLR (Ours)} & Joint & \textbf{77.5+3.2} & \textbf{81.9+2.3} & \textbf{68.3+4.9} & \textbf{70.1+6.7} & \textbf{83.2+0.2} \\ \hline

SkeletonCLR~\cite{Li2021CrossViewConsistency} & Motion & 53.3 & 50.8 & 39.6 & 40.2 & 63.4 \\
AimCLR~\cite{Guo2022AimCLR} & Motion & \underline{66.8} & \underline{70.6} & \underline{57.3} & \underline{54.4} & \underline{72.0} \\
\textbf{M\textsuperscript{3}GCLR (Ours)}  & Motion & \textbf{73.2+6.4} & \textbf{77.0+6.4} & \textbf{63.4+6.1} & \textbf{59.8+5.4} & \textbf{76.4+4.4} \\ \hline

SkeletonCLR~\cite{Li2021CrossViewConsistency} & Bone & 69.4 & 67.4 & 48.4 & 52.0 & 72.6 \\
AimCLR~\cite{Guo2022AimCLR} & Bone & \underline{73.2} & \underline{77.0} & \underline{62.9} & \underline{63.4} & \underline{82.0} \\
\textbf{M\textsuperscript{3}GCLR (Ours)}  & Bone & \textbf{76.1+2.9} & \textbf{78.9+1.9} & \textbf{65.8+2.9} & \textbf{64.2+0.8} & \textbf{85.2+3.2} \\ \hline

3s-SkeletonCLR~\cite{Li2021CrossViewConsistency} & 3s & 75.0 & 79.8 & 60.7 & 62.6 & 85.3 \\
3s-AimCLR~\cite{Guo2022AimCLR} & 3s & \underline{78.2} & \underline{83.8} & \underline{68.2} & \underline{68.8} & \underline{87.8} \\
\textbf{3s-M\textsuperscript{3}GCLR (Ours)}  & 3s & \textbf{82.1+3.9} & \textbf{85.8+2.0} & \textbf{72.3+4.1} & \textbf{75.0+6.2} & \textbf{89.1+1.3} \\ \hline

\end{tabular}
\end{table*}

The experiments are conducted on the NTU RGB+D 60 and PKU-MMD datasets for performance evaluation. To train the proposed network, all skeleton sequences are temporally downsampled to 50 frames. The encoder adopts ST-GCN~\cite{yan2018spatial} as the backbone network, with a hidden layer dimension of 256 and a feature dimension of 128. For the contrastive learning setting, we follow the configuration of the baseline method~\cite{Guo2022AimCLR}, where the projection heads for both contrastive learning and auxiliary tasks are implemented using multi-layer perceptrons (MLPs). For optimization, SGD with a momentum of 0.99 and a weight decay of 0.0001 is employed. The model is trained for 300 epochs, with an initial learning rate of 0.1, which is reduced to 0.01 at the 250th epoch. In addition, three streams of skeleton sequences are generated, including joint, bone, and motion streams. For all reported three-stream results, we adopt the same weighted fusion strategy as other multi-stream GCN-based methods~\cite{Guo2022AimCLR, Guo2024AimCLRPlus, Li2021CrossViewConsistency, Wang2023SS3D, Hua2023PartAwareCLR, Xu2021CapsuleAE, Chen2025VLSkeleton, Gao2023EfficientSTCL, zhang2023hierarchical}, with fusion weights set to [0.6, 0.6, 0.4].

\subsection{Performance Evaluation}
\subsubsection{Comparison with Baseline Methods}
We evaluate the proposed M\textsuperscript{3}GCLR under the linear evaluation protocol and conduct comparative experiments with representative baseline methods. M\textsuperscript{3}GCLR is compared with SkeletonCLR~\cite{Thoker2021SkeletonContrastive} and AimCLR~\cite{Guo2022AimCLR}, where the bolded entries in the result tables denote the best performance, and the underlined entries indicate the second-best results. As shown by the experimental results, M\textsuperscript{3}GCLR consistently outperforms the baseline methods across all evaluation metrics on both datasets.

As reported in Table~\ref{tab:base}, on the NTU RGB+D 60 dataset, the proposed method achieves significant improvements over existing baselines under both the single-stream (Joint, Bone and Motion) and three-stream (Joint+Bone+Motion) settings. For example, under the X-View protocol, the Joint-stream model attains an accuracy of 81.9\%, surpassing SkeletonCLR (76.4\%) and AimCLR (79.6\%) by 5.5\% and 2.3\%, respectively. This performance gain mainly attributed to the complementary effect of muti-view data augmentation. By jointly employing normal-augmentation rotation angle (\textpm15\textdegree) and extreme-augmentation rotation angle (\textpm60\textdegree), the model is able to simultaneously capture fine-grained local motion details (e.g., subtle finger movements) and global posture variations (e.g., body center-of-mass shifts), thereby enhancing feature representation capability. Moreover, the game-theoretic mechanism maximizes the mutual information between different views and static anchor points while minimizing redundancy across views, effectively suppressing noise interference. The mutual-information-driven optimization further encourages the model to focus on intrinsic action-related characteristics during training, rather than view-dependent noise. Additional experiments on the NTU RGB+D 120 and PKU-MMD Part I dataset show that our method consistently outperforms baseline methods, providing further validation of its effectiveness. 

Furthermore, in the three-stream training setting, the proposed method achieves accuracies of 85.8\%/82.1\% on NTU RGB+D 60 X-View/X-Sub, 75.0\%/72.3\% on NTU RGB+D 120 X-Set/X-Sub, and 89.1\% on PKU-MMD Part I, respectively. All results show consistent improvements over their single-stream counterparts. These results indicate that the collaborative effect of multi-modal information is effective: Bone data describe spatial relationships among joints, Motion data capture temporal dynamics, and both complement Joint data. The results also demonstrate the multi-stream adaptability of the game mechanism, showing that the min–max game framework remains robust in multi-stream scenarios. 

To visualize the effectiveness of the proposed method, we conducted the following visualization experiments. First, to intuitively demonstrate the effectiveness of the proposed method, we compared the linear evaluation results under the three-stream training setting with those of the baseline methods in a unified bar chart, as shown in Fig.~\ref{fig:8}. Next, we selected nine actions from the NTU RGB+D 60 dataset—drink water, brushing teeth, brushing hair, jumping up, clapping, reading, writing, cheer up, and hand waving—and plotted their t-SNE scatter distributions, as shown in Fig.~\ref{fig:6}. Finally, we randomly selected 100 samples for each of the nine actions and fed them into the trained model to predict their action categories. Based on the prediction results, a confusion matrix was generated, as illustrated in Fig.~\ref{fig:7}. 

\begin{table*}[!t]
\centering
\caption{Comparison Results of the Proposed M\textsuperscript{3}GCLR and Other Approaches on the NTU RGB+D 60 and NTU RGB+D 120 Datasets under the Linear Evaluation Protocol}
\label{tab:ntu}
\renewcommand{\arraystretch}{1.15}
\begin{tabular}{|c c c c c c|}
\hline
\multirow{2}{*}{Method} & 
\multirow{2}{*}{Backbone} & 
\multicolumn{2}{c}{NTU RGB+D 60 (\%)} & 
\multicolumn{2}{c|}{NTU RGB+D 120 (\%)} \\ \cline{3-6}
 & & X-Sub Acc.$\uparrow$ & X-View Acc.$\uparrow$ & X-Sub Acc.$\uparrow$ & X-Set Acc.$\uparrow$ \\ \hline
CorsSCLR (2021) \cite{Li2021CrossViewConsistency} & ST-GCN & 72.9 & 79.9 & 67.9 & 66.7 \\
MCAE (2021) \cite{Xu2021CapsuleAE} & MCAE & 65.6 & 74.7 & 52.8 & 54.7 \\
CP-STN (2021) \cite{Chen2025VLSkeleton} & ST-GCN & 69.4 & 76.6 & 55.7 & 54.7 \\
ST-CL (2022) \cite{Gao2023EfficientSTCL} & GCN & 68.1 & 69.4 & 54.2 & 55.6 \\
AimCLR (2022) \cite{Guo2022AimCLR} & ST-GCN & 78.2 & 83.8 & 68.2 & 68.8 \\
SkeAttnCLR (2023) \cite{Hua2023PartAwareCLR} & ST-GCN & 78.2 & 82.6 & 65.5 & 65.7 \\
ASAR (2023) \cite{Wang2023SS3D} & ST-GCN & 75.6 & 80.7 & 62.5 & 62.8 \\
HiCLR (2023) \cite{Dong2023HiCLR} & ST-GCN & 80.4 & 85.5 & 70.0 & 70.4 \\
ViA (2024) \cite{yang2024view} & GCN & 78.1 & \underline{85.8} & 69.2 & 66.9 \\
CMCS (2024) \cite{liu2024cross} & ST-GCN & 78.6 & 84.5 & 68.5 & 71.1 \\
AimCLR++ (2024) \cite{Guo2024AimCLRPlus} & ST-GCN & 80.9 & 85.4 & \underline{70.1} & 71.2 \\
ISSM (2025) \cite{tu2025issm} & bi-GRU & \underline{81.1} & \textbf{86.7} & 69.9 & \underline{73.2} \\
\textbf{M\textsuperscript{3}GCLR (Ours)} & ST-GCN & \textbf{82.1} & 85.5 & \textbf{72.3} & \textbf{75.0} \\ \hline
\end{tabular}
\end{table*}

\begin{table*}[!t]
\centering
\caption{Comparison Results of the Proposed M\textsuperscript{3}GCLR and Other Approaches on the PKU-MMD Part I and Part II Dataset under the Linear Evaluation Protocol}
\label{tab:pkummd}
\renewcommand{\arraystretch}{1.15}
\begin{tabular}{|c c c c|}
\hline
\multirow{2}{*}{Method} & 
\multirow{2}{*}{Backbone} & 
\multicolumn{2}{c|}{PKU-MMD (\%)} \\ \cline{3-4}
 &  & Part I Acc.$\uparrow$ & Part II Acc.$\uparrow$ \\ \hline
CorsSCLR (2021) \cite{Li2021CrossViewConsistency} & ST-GCN & 84.9 & 21.2 \\
AimCLR (2022) \cite{Guo2022AimCLR} & ST-GCN & 87.8 & 38.5 \\
SkeAttnCLR (2023) \cite{Hua2023PartAwareCLR} & ST-GCN & 83.8 & --- \\
ASAR (2023) \cite{Wang2023SS3D} & ST-GCN & 83.5 & 38.8 \\
CMCS (2024) \cite{liu2024cross} & ST-GCN & 88.1 & 39.6 \\
AimCLR++ (2024) \cite{Guo2024AimCLRPlus} & ST-GCN & \textbf{90.4} & \underline{41.2} \\ 
ISSM (2025) \cite{tu2025issm} & bi-GRU & \underline{89.6} & 40.3 \\
\textbf{M\textsuperscript{3}GCLR (Ours)} & ST-GCN & 89.1 & \textbf{45.2} \\ \hline
\end{tabular}
\end{table*}

\subsubsection{Comparison with State-of-the-Art Methods}
To further validate the effectiveness of the proposed model, we conduct comparative experiments with other self-supervised skeleton-based action recognition methods under the linear evaluation protocol. The results are reported in Tables~\ref{tab:ntu} and~\ref{tab:pkummd}, respectively. In the tables, the algorithm name in bold denotes the proposed method, and the bold metrics and underlined metrics respectively indicate the best and the second-best performance within each table. The year in parentheses represents the publication time of the corresponding method.

From the results in Table~\ref{tab:ntu}, it can be observed that M\textsuperscript{3}GCLR remains highly competitive with recent approaches on both the NTU RGB+D 60, NTU RGB+D 120 and PKU-MMD datasets. The core advantage lies in the introduction of the ISG, which explicitly disentangles view-related noise from intrinsic action representations. In contrast, existing methods mainly rely on standard and extreme data augmentations, which may lead to less comprehensive feature learning and lack finer theoretical constraints.

In Table~\ref{tab:pkummd}, PKU-MMD Part I is a relatively small and simple dataset. As a result, in this setting, the gap between the averaged data computed by M\textsuperscript{3}GCLR and the augmented data is limited, and the discriminative action information that can be extracted is therefore relatively constrained. This leads to performance that is slightly inferior to AimCLR++. Nevertheless, even under such disadvantageous conditions, M\textsuperscript{3}GCLR remains close to the best-performing method and still substantially outperforms earlier approaches, which further demonstrates its superiority. The training results on PKU-MMD Part II show that the proposed ISG plays a significant role in breaking the upper bound of model performance, thereby validating the effectiveness of our method.

\subsection{Ablation Study}
M\textsuperscript{3}GCLR integrates MRAM, M\textsuperscript{3}ISGM, and DLEO. To evaluate the contribution of each component, we design a progressive ablation study.

Table~\ref{tab:abs} shows the influence of different augmentation strategies on model performance. With only normal augmentation (74.9\%), small-angle rotations preserve local details but lack viewpoint diversity, making the model sensitive to global posture variations. With only multi-view augmentation (71.2\%), large-angle rotations increase diversity but excessively disrupt action continuity (e.g., the hand trajectory in waving becomes fragmented), resulting in semantic inconsistency among positive samples. When combining normal+multi-view-rotation augmentation (77.4\%), the complementarity between local and global features is partially exploited; however, redundant information is not explicitly constrained, leading to limited improvement. After introducing the mutual information constraint (82.1\%), the ISG minimizes inter-view mutual information, enabling the model to effectively filter redundant features (e.g., noise from background joints), and the accuracy improves significantly by 4.7\%. These results suggest that MRAM can achieve substantial performance gains only when collaborating with M\textsuperscript{3}ISGM and DLEO.

\begin{table*}[!t]
\centering
\caption{Model-Wise Ablation Results of M\textsuperscript{3}GCLR on the NTU RGB+D 60 Dataset (Linear Evaluation)}
\label{tab:abs}
\renewcommand{\arraystretch}{1.15}
\begin{tabular}{|c c c c c|}
\hline
\multirow{2}{*}{Normal Augmentation} & 
\multirow{2}{*}{MRAM} & 
\multirow{2}{*}{M\textsuperscript{3}ISGM \& DLEO} & 
\multicolumn{2}{c|}{NTU RGB+D 60 (\%)} \\ \cline{4-5}
 &  &  & X-Sub Acc.$\uparrow$ & X-View Acc.$\uparrow$ \\ \hline
\checkmark &   &   & 74.9 & 79.5 \\
 & \checkmark  &   & 71.2 & 77.1 \\
\checkmark & \checkmark &   & \underline{77.4} & \underline{82.5} \\
\checkmark & \checkmark & \checkmark & \textbf{82.1} & \textbf{85.5} \\ \hline
\end{tabular}
\end{table*}

\begin{table*}[!t]
\centering
\caption{Ablation Results of the Three Streams under Different Rotation Angle Ranges on the NTU RGB+D 60 X-Sub Dataset (Linear Evaluation)}
\label{tab:angle}
\renewcommand{\arraystretch}{1.15}
\begin{tabular}{|c c c c c|}
\hline
Rotation Angle Range (\textdegree) & [--15, 15] & [--30, 30] & [--60, 60] & [--120, 120] \\ \hline
Normal & \textbf{82.1} & \underline{80.2} & 75.4 & 75.3 \\ 
Extreme & 70.1 & 71.3 & \textbf{82.1} & \underline{78.1} \\ \hline
\end{tabular}
\end{table*}

Furthermore, to investigate the effects of rotation angles and scales, as well as how different multi-view settings influence the learning behavior of M\textsuperscript{3}GCLR, we conduct parameter studies, as reported in Table~\ref{tab:angle}. We first fix the angle of normal rotation at [--15\textdegree, 15\textdegree] and examine the impact of different extreme-view rotations along the x-, y-, and z-axes. Then we fix the angle of extreme rotation at [--60\textdegree, 60\textdegree] and examine the impact of different normal-view rotations along the x-, y-, and z-axes.

It can be observed that when the extreme rotation angle is identical to the normal rotation angle, the additional multi-view information cannot be effectively exploited. As a result, the performance remains close to that of the baseline, while extra computational cost is introduced. As the extreme rotation magnitude increases, the model gradually benefits from the mutual information across different views, leading to moderate performance improvement. When the normal and extreme rotation ranges are set to [--15\textdegree, 15\textdegree] and [--60\textdegree, 60\textdegree], respectively, the model achieves favorable results. However, if the rotation range is further enlarged, excessive transformations lead to feature loss, and the performance begins to deteriorate. The sensitivity analysis of extreme rotation angles suggests that moderate rotational augmentation can simulate realistic viewpoint variations (e.g., top or side views) while preserving semantic integrity. In contrast, overly large rotations distort joint coordinates (for example, causing overlaps between arms and the torso), disrupt spatiotemporal continuity, and produce misleading positive samples. This observation motivates future investigation into appropriate threshold ranges that yield optimal performance.


\section{Conclusion}

This paper addresses the common challenges in cross-view skeleton-based action recognition, including view sensitivity, uncontrollable augmentation perturbations, and the lack of effective adversarial mechanisms. Grounded in the existence theory of game-theoretic equilibrium and the mini-max principle, we formulate a mini-max Infinite Skeleton-data Game (ISG) and, based on this formulation, propose a contrastive learning framework termed M\textsuperscript{3}GCLR (Multi-view Mini-Max infinite skeleton-data game Contrastive Learning for skeleton-based action Recognition). First, we design a Multi-view Rotation-based Augmentation Module (MRAM) to enhance the discriminative capability of M\textsuperscript{3}GCLR and improve its adaptability to viewpoint variations. Second, we introduce a Mutual-information-based Mini-Max skeleton-data Game module (M\textsuperscript{3}ISGM), which maximizes the semantic consistency between individual samples and the mean representation while reducing discrepancies across multi-view data, thereby providing a strong adversarial constraint. Building upon this, we further develop a Dual-Loss-based Equilibrium Optimizer (DLEO) to facilitate the convergence toward game equilibrium and improve feature separability. To evaluate the effectiveness of M\textsuperscript{3}GCLR, we conduct extensive experiments on NTU RGB+D 60, NTU RGB+D 120, and PKU-MMD Part I \& II. The results show that our method not only outperforms baseline approaches but also achieves competitive performance against recent state-of-the-art methods. In addition, ablation studies verify the effectiveness of MRAM, M\textsuperscript{3}ISGM, and DLEO, demonstrating the superiority and generality of the proposed framework in enhancing view robustness.

\clearpage

\appendix
\setcounter{equation}{0}
\renewcommand{\theequation}{A-\arabic{equation}}
\setcounter{table}{0}
\renewcommand{\thetable}{A-\Roman{table}}
\setcounter{figure}{0}
\renewcommand{\thefigure}{A-\arabic{figure}}

\renewcommand{\theHequation}{A.\arabic{equation}}
\renewcommand{\theHtable}{A.\Roman{table}}
\renewcommand{\theHfigure}{A.\Roman{figure}}

\section*{Proof of Theorems}
\label{ap:proof}

We first list all significant mathematical symbols in Appendix in Table~\ref{tab:symbols_a}.
\begin{table*}[!t]
\centering
\caption{Significant Mathematical Symbols in Appendix}
\label{tab:symbols_a}
\renewcommand{\arraystretch}{1.25}

\begingroup
\setlength{\parindent}{0pt}

\begin{tabular}{|>{\centering\arraybackslash}m{3.7cm}|>{\justifying\arraybackslash}m{10.2cm}|}
\hline
\multicolumn{1}{|c|}{Symbol} & \multicolumn{1}{c|}{Description} \\
\hline
$\Delta(\Omega)$ &
\noindent The set of all probability distributions over the sample space $\Omega$. \\

$\times$ (or $\times_{i \in N}$) &
\noindent The Cartesian product of two or more sets, representing all possible ordered tuples. Example: if $A=\{1\}$ and $B=\{2,3\}$, then $A \times B = \{(1,2),(1,3)\}$. \\

$\Gamma = (N,(C_i)_{i\in N},(u_i)_{i\in N})$ &
\noindent A strategic-form game $\Gamma$, where $N$ is the set of players, $C_i$ is the strategy set of player $i$, and $u_i$ is the utility (or payoff) function of player $i$. \\

$[\cdot]$ &
\noindent Pure strategy in a strategic game. \\

$-j$ &
\noindent The remaining set after removing element $j$ from $N$, i.e., $-j = N\setminus\{j\}$.
Example: if $N=\{1,2,3\}$, then $-2=\{1,3\}$. \\

$\sigma_i$ &
\noindent A randomized strategy of player $i$. Example: if $C_1=\{[a_1],[b_1]\}$, then $\Delta(C_1)=\{p[a_1] + (1-p)[b_1] \mid 0\le p\le1\}$,
and $\sigma_1 = 0.5[a_1] + 0.5[b_1]$. \\

$\sigma_i(c_i)$ &
\noindent The probability assigned to pure strategy $c_i$ by randomized strategy $\sigma_i$. Example: $\sigma_1([a_1]) = 0.5$. \\

$\sigma_S$ &
\noindent The joint randomized strategy profile of all players in subset $S\subset N$. When $S=N$, the subscript can be omitted.
Example: $\sigma_1 = 0.5[a_1] + 0.5[b_1]$, $\sigma_2 = [a_2]$,
then $\sigma_{\{1,2\}} = (\sigma_1, \sigma_2)$. \\

$I(\cdot; \cdot)$ & \noindent The average mutual information between two samples. \\

$d(\cdot, \cdot)$ & \noindent The Euclid distance between two points. \\
\hline
\end{tabular}

\endgroup
\end{table*}

\begin{proof}[Proof of Theorem 1]
\label{pr:th1}
To proof Theorem~\ref{the:isg}, we should introduce a supporting lemma.

\begin{lemma}[Kakutani Fixed Point Theorem~\cite{kakutani1941}]
\label{le:kakutani}
Let $\mathbf{S}$ be a nonempty, convex, and compact subset of a finite-dimensional vector space. If a mapping $F: \mathbf{S} \mapsto \mathbf{S}$ is upper hemicontinuous, then there exists at least one point $\bar{\mathbf{x}} \in \mathbf{S}$ such that $\bar{\mathbf{x}} \in F(\bar{\mathbf{x}})$. Here, an upper hemicontinuous mapping $F: \mathbf{X} \mapsto \mathbf{Y}$ satisfies the following condition: whenever (a) for all positive integers $j$, $\mathbf{x}(j) \in \mathbf{X}$, $\mathbf{y}(j) \in F(\mathbf{x}(j))$; (b) $\mathbf{x}(j) \xrightarrow{j \to \infty} \bar{\mathbf{x}}$; (c) $\mathbf{y}(j) \xrightarrow{j \to \infty} \bar{\mathbf{y}}$, it follows that $\bar{\mathbf{y}} \in F(\bar{\mathbf{x}})$.
\end{lemma}

We first prove that when the utility functions $u_i$ of the strategic-form game $\Gamma = (N,(C_i)_{i \in N},(u_i)_{i \in N})$ are continuous and each $C_i$ is compact and bounded, an equilibrium of $\Gamma$ exists.

Since each $C_i$ is compact and bounded, by the mathematical properties of probability distribution sets and Cartesian products, $\times_{i \in N} \Delta(C_i)$ is also compact and bounded. Moreover, as $\Delta(C_i)$ is convex and the Cartesian product of convex sets preserves convexity, $\times_{i \in N} \Delta(C_i)$ is also convex. For any randomized strategy profile $\sigma \in \times_{i \in N} \Delta(C_i)$ and any player $j \in N$, define:
\begin{equation}
\label{eq:a-1}
R_j(\sigma_{-j}) = \mathop{\arg\max}_{\tau_j \in \Delta(C_j)} u_j(\tau_j, \sigma_{-j}).
\end{equation}
Thus, the set $R_j(\sigma_{-j})$ represents the best response set for player $j$ in $\Delta(C_j)$ with respect to all other players $N \setminus \{j\}$. According to the definition of Nash equilibrium (see Eq.~(\ref{eq:2}) in the main text), the set $R_j(\sigma_{-j})$ can be expressed as a probability distribution set, where
\begin{equation}
\label{eq:a-2}
r_j(c_j) = 0 \Leftrightarrow c_j \notin \mathop{\arg\max}_{d_j \in C_j}u_j(\sigma_{-j}, d_j).
\end{equation}
Here, $r_j \in R_j(\sigma_{-j})$. This implies that $R_j(\sigma_{-j})$ is a subset of $\Delta(C_j)$ defined by linear equalities. Based on fundamental logical reasoning, if $c_j \in R_j(\sigma_{-j})$, then it must hold that
\begin{equation}
\label{eq:a-3}
c_j \in \mathop{\arg\max}_{d_j \in C_j}u_j(\sigma_{-j}, d_j).
\end{equation}
Since that utility function $u_j(\cdot)$ is continuous and its domain is compact, it follows that: $\mathop{\arg\max}_{d_j \in C_j} u_j(\sigma_{-j}, d_j) \neq \varnothing$. Consequently, by the definition in Eq.~(\ref{eq:a-3}), such an element $c_j$ must exist. Therefore, the best-response set $R_j(\sigma_{-j})$ contains at least one element, i.e., $R_j(\sigma_{-j}) \neq \varnothing$. As a result, the set $\times_{i \in N} \Delta(C_i)$ is nonempty. Hence, the set $\times_{i \in N} \Delta(C_i)$ satisfies the properties of being nonempty, convex, bounded, and closed.

Next, define $R(\sigma) = \times_{i \in N} R_i(\sigma_{-i})$. According to the definition of Nash equilibrium, $R(\sigma)$ represents the set of all Nash equilibria of the game $\Gamma$. If there exists a convergent sequence of randomized stratege profiles $\{\sigma^{k}\}$, where for any $k$, $\sigma^k \in \times_{i \in N} \Delta(C_i)$, such that $\tau^k \in R(\sigma^k)$ (which must exist since $R(\sigma^k)$ is nonempty), then by the continuity of the utility function $u_i(\cdot)$, one can always extract a convergent subsequence ${\tau^k}$ satisfying:
\begin{equation}
\label{eq:a-4}
\bar\tau = \lim_{k \rightarrow \infty} \tau^k,
\end{equation}
and
\begin{equation}
\label{eq:a-5}
\bar\sigma = \lim_{k \rightarrow \infty} \sigma^k.
\end{equation}

Note that $\tau^k$ is a sequence of Nash equilibria of the game $\Gamma$. Therefore, for any $\rho_i \in \Delta(C_i)$, the following inequality holds:
\begin{equation}
\label{eq:a-6}
u_i(\sigma_{-i}^k, \tau_i^k) \ge u_i(\sigma_{-i}^k, \rho_i).
\end{equation}
By the preservation property of limits, it follows that
\begin{equation}
\label{eq:a-7}
u_i(\bar{\sigma}_{-i}, \bar{\tau}_i) \ge u_i(\bar{\sigma}_{-i}, \rho_i).
\end{equation}
Hence, $\bar{\tau} = (\bar{\tau}_i)_{i \in N}$ is a Nash equilibrium of the game $\Gamma$, which implies $\bar{\tau} \in R(\bar\sigma)$. As a result, the correspondence $R(\sigma)$ is upper hemicontinuous.

Therefore, by the lemma~\ref{le:kakutani}, there exists a $\sigma \in \times_{i \in N} \Delta(C_i)$ such that $\sigma \in R(\sigma)$. That is, the randomized strategy profile $\sigma$ is a Nash equilibrium of the game $\Gamma$, and hence the Nash equilibrium of $\Gamma$ is guaranteed to exist. 

Next, we prove that when the Infinite Skeleton-data Game (ISG) $\Gamma_S = (E, (\boldsymbol{\uptheta_i})_{i \in N}, (u_i(\boldsymbol\uptheta))_{i \in N})$ admits a polynomial function of mutual information-based utility function $u_i$, and the space $\boldsymbol\Theta_i$ to which $\boldsymbol\uptheta_i$ belongs is compact and bounded, the equilibrium of $\Gamma_S$ necessarily exists. Here, $\phi(\boldsymbol\uptheta_i) = \sigma_i$, where $\phi$ is a continuous mapping.

Consider the strategic-form game $\Gamma = (N,(C_i)_{i \in N}, (v_i)_{i \in N})$, where the player set satisfies $N = E$, $\sigma_i = \phi(\boldsymbol\uptheta_i) \in \Delta(C_i)$ for all $i \in E(= N)$, and the payoff  functions $v_i$ are defined as $v_i(\sigma) = u_i(\boldsymbol\uptheta)$. Since each $\boldsymbol\Theta_i$ is compact and bounded, $\phi$ is continuous, and the construction in Eq.~(\ref{eq:a-3}) holds, it follows that $\times_{i \in N} \Delta(C_i)$ is nonempty. Therefore, $\times_{i \in N} \Delta(C_i)$ is nonempty, convex, bounded, and closed. Moreover, since each $u_i(\boldsymbol\uptheta)$ is a polynomial function of mutual information function and mutual information is continuous, the corresponding payoff function $v_i(\sigma)$ is also continuous. Hence, the game $\Gamma$ admits at least one Nash equilibrium $\sigma^*$. By letting $\sigma^* = \phi(\boldsymbol\uptheta^*)$, we immediately obtain an equilibrium $\boldsymbol\uptheta^*$ of the ISG $\Gamma_S$. Therefore, when the utility functions $u_i$ are mutual information-based and the parameter space $\boldsymbol\Theta_i$ is compact and bounded, the equilibrium of the ISG $\Gamma_S$ is guaranteed to exist.
\end{proof}

\begin{proof}[Proof of the Equivalence between DLEO and M\textsuperscript{3}ISGM]
\label{pr:th3}
The equivalence between DLEO and M\textsuperscript{3}ISGM is established as follows. Since the game between encoder 1 and encoder 2 is formulated as a mini-max ISG, according to Theorem~\ref{the:m2} (see Section~\ref{sec:isg}), it suffices to analyze the Nash equilibrium by examining the utility function $u_1$ of encoder 1. The Nash equilibrium $(\boldsymbol\uptheta_1, \boldsymbol\uptheta_2)$ satisfies:
\begin{equation}
\label{eq:a-8}
\begin{aligned}
u_1(\boldsymbol\uptheta_1, \boldsymbol\uptheta_2) =& \mathop{\max}_{\hat{\mathbf{z}}^{(i)}} \mathop{\min}_{\tilde{\mathbf{z}}^{(i)}} (-I(\hat{\mathbf{z}}^{(i)}; \bar{\mathbf{z}}^{(i)})\\ 
&- [I(\hat{\mathbf{z}}^{(i)}; \bar{\mathbf{z}}^{(i)}) - I(\tilde{\mathbf{z}}^{(i)}; \bar{\mathbf{z}}^{(i)})]^2),
\end{aligned}
\end{equation}
where, with respect to $\hat{\mathbf{z}}^{(i)}$, the objective function is an upward-opening quadratic function. Therefore, the function attains its minimum iff $I(\hat{\mathbf{z}}^{(i)}; \bar{\mathbf{z}}^{(i)}) = I(\tilde{\mathbf{z}}^{(i)}; \bar{\mathbf{z}}^{(i)})$, in which case the minimum value with respect to $\hat{\mathbf{z}}^{(i)}$ is $-I(\hat{\mathbf{z}}^{(i)}; \bar{\mathbf{z}}^{(i)})$. Furthermore, according to the definition of average mutual information, the value of $I(\hat{\mathbf{z}}^{(i)}; \bar{\mathbf{z}}^{(i)})$ is strictly negatively correlated with the distance $d(\hat{\mathbf{z}}^{(i)}; \bar{\mathbf{z}}^{(i)})$. That is, maximizing $-I(\hat{\mathbf{z}}^{(i)}; \bar{\mathbf{z}}^{(i)})$ is equivalent to maximizing the distance $d(\hat{\mathbf{z}}^{(i)}; \bar{\mathbf{z}}^{(i)})$.

\begin{figure}[t]
  \begin{center}
  \includegraphics[width=0.4\columnwidth]{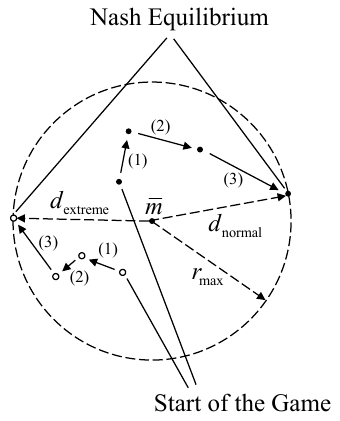}\\
  \caption{Illustration of the convergence process of the mini-max ISG toward equilibrium. }\label{fig:a-1}
  \end{center}
\end{figure}

Therefore, all Nash equilibria $(\boldsymbol\uptheta_1, \boldsymbol\uptheta_2)$ of the ISG $\Gamma_S$ must satisfy that $(\hat{\mathbf{z}}^{(i)}, \tilde{\mathbf{z}}^{(i)})$ fulfills:
\begin{equation}
\label{eq:a-9}
\hat{\mathbf{z}}^{(i)} \in \mathop{\arg\max}_{\mathbf{z} \in \mathcal{D}_{\hat{\mathbf{z}}}} d(\mathbf{z}, \bar{\mathbf{z}}^{(i)}), \quad{d(\hat{\mathbf{z}}^{(i)}; \bar{\mathbf{z}}^{(i)}) = d(\tilde{\mathbf{z}}^{(i)}; \bar{\mathbf{z}}^{(i)}), }
\end{equation}
where $\hat{\mathbf{z}}^{(i)} = f_{q\_1}(\hat{\mathbf{X}}^{(i)})$, $\tilde{\mathbf{z}}^{(i)} = f_{q\_2}(\tilde{\mathbf{X}}^{(i)})$, and $\mathcal{D}_{\hat{\mathbf{z}}}$ denotes the value domain of $\hat{\mathbf{z}}^{(i)}$. Accordingly, one Nash equilibrium of the proposed M\textsuperscript{3}ISGM can be obtained, as illustrated in Fig.~\ref{fig:a-1}.

In Fig.~\ref{fig:a-1}, a solid dot without annotation denotes the coordinate of a feature vector $\hat{\mathbf{z}}^{(i)}$, while a hollow dot denotes the coordinate of $\tilde{\mathbf{z}}^{(i)}$. During the game process, encoder 1 tends to push $\hat{\mathbf{z}}^{(i)}$ away from $\bar{\mathbf{z}}^{(i)}$ as far as possible in order to maximize the utility $u_1$. In contrast, encoder 2 attempts to adjust the distance between $\tilde{\mathbf{z}}^{(i)}$ and $\bar{\mathbf{z}}^{(i)}$ so that it approaches the distance between $\hat{\mathbf{z}}^{(i)}$ and $\bar{\mathbf{z}}^{(i)}$, thereby reducing $u_1$ while increasing $u_2$. After three iterations of the game, since $I(\hat{\mathbf{z}}^{(i)};\bar{\mathbf{z}}^{(i)})$ reaches its minimum and simultaneously satisfies
$I(\hat{\mathbf{z}}^{(i)};\bar{\mathbf{z}}^{(i)}) = I(\tilde{\mathbf{z}}^{(i)};\bar{\mathbf{z}}^{(i)})$,
the game $\Gamma_S$ converges to an equilibrium state.

\begin{figure}[t]
  \begin{center}
  \includegraphics[width=0.55\columnwidth]{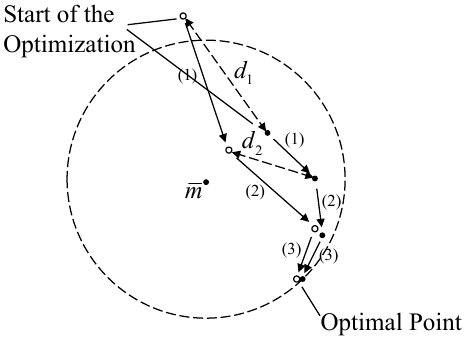}\\
  \caption{Illustration of the optimization process of DLEO. }\label{fig:a-2}
  \end{center}
\end{figure}

However, it can be observed that if the game is defined in this manner alone, it is difficult to guarantee that $\tilde{\mathbf{z}}^{(i)}$ is sufficiently close to $\hat{\mathbf{z}}^{(i)}$, and the game can only ensure that the degree to which $\hat{\mathbf{z}}^{(i)}$ moves away from $\bar{\mathbf{z}}^{(i)}$ is comparable to that of $\tilde{\mathbf{z}}^{(i)}$. Nevertheless, when $\hat{\mathbf{z}}^{(i)}$ and $\tilde{\mathbf{z}}^{(i)}$ are sufficiently close to each other and at a comparable distance from $\bar{\mathbf{z}}^{(i)}$, it is evident that this configuration also constitutes a Nash equilibrium of the game. This observation provides an important insight: if appropriate constraints can be deliberately introduced, the game can be guided to evolve toward a desired equilibrium. Such a mechanism is a concrete manifestation of the focal effect~\cite{rapoport2012game} in game theory.

Since both $\mathcal{L}_1$ and $\mathcal{L}_2$ are convex functions, the two optimization problems involved in DLEO are convex optimization problems, and therefore each admits a global optimum. For $\mathcal{L}_1$, according to the definitions of $\mathcal{L}_{\mathrm{Push}}$ and the KL divergence, its optimum satisfies:
\begin{equation}
\label{eq:a-10}
\hat{\mathbf{z}}^{(i)} \in \mathop{\arg\max}_{\mathbf{z} \in \mathcal{D}_{\hat{\mathbf{z}}}} d(\mathbf{z}, \bar{\mathbf{z}}^{(i)}), \quad{\tilde{\mathbf{z}}^{(i)} \in \mathop{\arg\min}_{\mathbf{z} \in \mathcal{D}_{\tilde{\mathbf{z}}}} d(\mathbf{z}, \bar{\mathbf{z}}^{(i)})},
\end{equation}
where $\mathcal{D}_{\tilde{\mathbf{z}}}$ denotes the domain of $\tilde{\mathbf{z}}^{(i)}$. Moreover, since the maximum rotation angle of the extremely augmented features is larger than that of the normally augmented features, it follows that $\mathop{\min}_{\mathbf{z} \in \mathcal{D}_{\tilde{\mathbf{z}}}} d(\mathbf{z}, \hat{\mathbf{z}}^{(i)}) = 0$, iff $z = \hat{\mathbf{z}}^{(i)}$. Consequently, the optimal solution is given by
\begin{equation}
\label{eq:a-11}
\hat{\mathbf{z}}^{(i)} \in \mathop{\arg\max}_{\mathbf{z} \in \mathcal{D}_{\hat{\mathbf{z}}}} d(\mathbf{z}, \bar{\mathbf{z}}^{(i)}), \quad{\tilde{\mathbf{z}}^{(i)} = \hat{\mathbf{z}}^{(i)}}.
\end{equation}
It can be observed that Eq.~(\ref{eq:a-11}) is a special case of Eq.~(\ref{eq:a-9}), and thus necessarily corresponds to one of the Nash equilibria in M\textsuperscript{3}ISGM. This implies that, in order to obtain a Nash equilibrium of M\textsuperscript{3}ISGM that satisfies the requirements of the proposed network, it suffices to solve the optimization problem in DLEO.

For intuitive illustration, Fig.~\ref{fig:a-2} is used to depict the optimization process. The meanings of all elements in the figure are the same as those in Fig.~\ref{fig:a-1}.
\end{proof}


%





\ifCLASSOPTIONcaptionsoff
  \newpage
\fi





\bibliographystyle{IEEEtran}
\bibliography{IEEEabrv,Bibliography}

@article{han2013enhanced,
  title={Enhanced computer vision with microsoft kinect sensor: A review},
  author={Han, Jungong and Shao, Ling and Xu, Dong and Shotton, Jamie},
  journal={IEEE transactions on cybernetics},
  volume={43},
  number={5},
  pages={1318--1334},
  year={2013},
  publisher={IEEE}
}

@article{zheng2023deep,
  title={Deep learning-based human pose estimation: A survey},
  author={Zheng, Ce and Wu, Wenhan and Chen, Chen and Yang, Taojiannan and Zhu, Sijie and Shen, Ju and Kehtarnavaz, Nasser and Shah, Mubarak},
  journal={ACM computing surveys},
  volume={56},
  number={1},
  pages={1--37},
  year={2023},
  publisher={ACM New York, NY}
}

@article{Su2019PredictCluster,
  author    = {K. Su and X. Liu and E. Shlizerman},
  title     = {PREDICT \& CLUSTER: Unsupervised Skeleton Based Action Recognition},
  journal   = {arXiv preprint arXiv:1911.12409},
  year      = {2019},
  doi       = {10.48550/ARXIV.1911.12409}
}

@inproceedings{Lin2020MS2L,
  author    = {L. Lin and S. Song and W. Yang and J. Liu},
  title     = {MS2L: Multi-Task Self-Supervised Learning for Skeleton Based Action Recognition},
  booktitle = {Proceedings of the 28th ACM International Conference on Multimedia},
  pages     = {2490--2498},
  year      = {2020},
  doi       = {10.1145/3394171.3413548}
}

@inproceedings{Zheng2018Unsupervised,
  author    = {N. Zheng and J. Wen and R. Liu and L. Long and J. Dai and Z. Gong},
  title     = {Unsupervised Representation Learning With Long-Term Dynamics for Skeleton Based Action Recognition},
  booktitle = {AAAI},
  volume    = {32},
  number    = {1},
  year      = {2018},
  doi       = {10.1609/aaai.v32i1.11853},
  pages     = {2644--2651}
}

@inproceedings{Thoker2021SkeletonContrastive,
  author    = {F. M. Thoker and H. Doughty and C. G. M. Snoek},
  title     = {Skeleton-Contrastive 3D Action Representation Learning},
  booktitle = {Proceedings of the 29th ACM International Conference on Multimedia},
  pages     = {1655--1663},
  year      = {2021},
  doi       = {10.1145/3474085.3475307}
}

@article{He2019MoCo,
  author    = {K. He and H. Fan and Y. Wu and S. Xie and R. Girshick},
  title     = {Momentum Contrast for Unsupervised Visual Representation Learning},
  journal   = {arXiv preprint arXiv:1911.05722},
  year      = {2019},
  doi       = {10.48550/ARXIV.1911.05722}
}

@inproceedings{Guo2022AimCLR,
  author    = {T. Guo and H. Liu and Z. Chen and M. Liu and T. Wang and R. Ding},
  title     = {Contrastive Learning from Extremely Augmented Skeleton Sequences for Self-Supervised Action Recognition},
  booktitle = {AAAI},
  volume    = {36},
  number    = {1},
  pages     = {762--770},
  year      = {2022},
  doi       = {10.1609/aaai.v36i1.19957}
}

@article{Lin2023ActionletDependent,
  author    = {L. Lin and J. Zhang and J. Liu},
  title     = {Actionlet-Dependent Contrastive Learning for Unsupervised Skeleton-Based Action Recognition},
  journal   = {arXiv preprint arXiv:2303.10904},
  year      = {2023},
  doi       = {10.48550/ARXIV.2303.10904}
}

@inproceedings{Gao2021ContrastiveSSL,
  title={Contrastive self-supervised learning for skeleton action recognition},
  author={Gao, Xuehao and Yang, Yang and Du, Shaoyi},
  booktitle={NeurIPS 2020 workshop on pre-registration in machine learning},
  pages={51--61},
  year={2021},
  organization={PMLR}
}

@article{Guo2024AimCLRPlus,
  title={Improving self-supervised action recognition from extremely augmented skeleton sequences},
  author={Guo, Tianyu and Liu, Mengyuan and Liu, Hong and Wang, Guoquan and Li, Wenhao},
  journal={Pattern Recognition},
  volume={150},
  pages={110333},
  year={2024},
}

@book{rapoport2012game,
  title={Game theory as a theory of conflict resolution},
  author={Rapoport, Anatol},
  volume={2},
  year={2012},
  publisher={Springer Science \& Business Media}
}

@article{Li2020ShapeMotion,
  title={Learning shape and motion representations for view invariant skeleton-based action recognition},
  author={Li, Yanshan and Xia, Rongjie and Liu, Xing},
  journal={Pattern Recognition},
  volume={103},
  pages={107293},
  year={2020}
}

@article{Rao2021MomentumLSTM,
  author    = {H. Rao and S. Xu and X. Hu and J. Cheng and B. Hu},
  title     = {Augmented Skeleton Based Contrastive Action Learning with Momentum LSTM for Unsupervised Action Recognition},
  journal   = {Information Sciences},
  volume    = {569},
  pages     = {90--109},
  year      = {2021},
  doi       = {10.1016/j.ins.2021.04.023}
}

@article{Li2021CrossViewConsistency,
  author    = {L. Li and M. Wang and B. Ni and H. Wang and J. Yang and W. Zhang},
  title     = {3D Human Action Representation Learning via Cross-View Consistency Pursuit},
  journal   = {arXiv preprint arXiv:2104.14466},
  year      = {2021},
  doi       = {10.48550/ARXIV.2104.14466}
}

@article{Liu2021AdaptiveMultiView,
  author    = {X. Liu and Y. Li and R. Xia},
  title     = {Adaptive Multi-View Graph Convolutional Networks for Skeleton-Based Action Recognition},
  journal   = {Neurocomputing},
  volume    = {444},
  pages     = {288--300},
  year      = {2021},
  doi       = {10.1016/j.neucom.2020.03.126}
}

@article{Chen2022MixedSkeleton,
  author    = {Z. Chen and H. Liu and T. Guo and Z. Chen and P. Song and H. Tang},
  title     = {Contrastive Learning from Spatio-Temporal Mixed Skeleton Sequences for Self-Supervised Skeleton-Based Action Recognition},
  journal   = {arXiv preprint arXiv:2207.03065},
  year      = {2022},
  doi       = {10.48550/ARXIV.2207.03065}
}

@article{Xia2022LAGANet,
  author    = {R. Xia and Y. Li and W. Luo},
  title     = {LAGA-Net: Local-and-Global Attention Network for Skeleton Based Action Recognition},
  journal   = {IEEE Transactions on Multimedia},
  volume    = {24},
  pages     = {2648--2661},
  year      = {2022},
  doi       = {10.1109/TMM.2021.3086758}
}

@inproceedings{Dong2023HiCLR,
  author    = {J. Dong and S. Sun and Z. Liu and S. Chen and B. Liu and X. Wang},
  title     = {Hierarchical Contrast for Unsupervised Skeleton-Based Action Representation Learning},
  booktitle = {AAAI},
  volume    = {37},
  number    = {1},
  pages     = {525--533},
  year      = {2023},
  doi       = {10.1609/aaai.v37i1.25127}
}

@inproceedings{Wu2024SCDNet,
  author    = {C. Wu and others},
  title     = {SCD-Net: Spatiotemporal Clues Disentanglement Network for Self-Supervised Skeleton-Based Action Recognition},
  booktitle = {AAAI},
  volume    = {38},
  number    = {6},
  pages     = {5949--5957},
  year      = {2024},
  doi       = {10.1609/aaai.v38i6.28409}
}

@article{Hua2023PartAwareCLR,
  author    = {Y. Hua and others},
  title     = {Part Aware Contrastive Learning for Self-Supervised Action Recognition},
  journal   = {arXiv preprint arXiv:2305.00666},
  year      = {2023},
  doi       = {10.48550/ARXIV.2305.00666}
}

@inproceedings{Wang2023SS3D,
  author    = {G. Wang and H. Liu and T. Guo and J. Guo and T. Wang and Y. Li},
  title     = {Self-Supervised 3D Skeleton Representation Learning with Active Sampling and Adaptive Relabeling for Action Recognition},
  booktitle = {IEEE International Conference on Image Processing (ICIP)},
  pages     = {56--60},
  year      = {2023},
  doi       = {10.1109/ICIP49359.2023.10221961}
}

@article{Li2022MDLBP,
  author    = {Y. Li and H. Tang and W. Xie and W. Luo},
  title     = {Multidimensional Local Binary Pattern for Hyperspectral Image Classification},
  journal   = {IEEE Transactions on Geoscience and Remote Sensing},
  volume    = {60},
  pages     = {1--13},
  year      = {2022},
  doi       = {10.1109/TGRS.2021.3069505}
}

@article{Mao2022CMD,
  author    = {Y. Mao and W. Zhou and Z. Lu and J. Deng and H. Li},
  title     = {CMD: Self-Supervised 3D Action Representation Learning with Cross-modal Mutual Distillation},
  journal   = {arXiv preprint arXiv:2208.12448},
  year      = {2022},
  doi       = {10.48550/ARXIV.2208.12448}
}

@article{Yu2024GeoExplainer,
  author    = {R. Yu and Y. Li and H. Liang and Z. Chen},
  title     = {GeoExplainer: Interpreting Graph Convolutional Networks with Geometric Masking},
  journal   = {Neurocomputing},
  volume    = {605},
  pages     = {128393},
  year      = {2024},
  doi       = {10.1016/j.neucom.2024.128393}
}

@article{Li2024BICAM,
  author    = {Y. Li and H. Liang and R. Yu},
  title     = {BI-CAM: Generating Explanations for Deep Neural Networks Using Bipolar Information},
  journal   = {IEEE Transactions on Multimedia},
  volume    = {26},
  pages     = {568--580},
  year      = {2024},
  doi       = {10.1109/TMM.2023.3267884}
}

@article{Li2024GTCAM,
  author    = {Y. Li and T. Shi and Z. Chen and L. Zhang and W. Xie},
  title     = {GT-CAM: Game Theory Based Class Activation Map for GCN},
  journal   = {IEEE Transactions on Pattern Analysis and Machine Intelligence},
  volume    = {46},
  number    = {12},
  pages     = {8806--8819},
  year      = {2024},
  doi       = {10.1109/TPAMI.2024.3413026}
}

@article{wang2022contrastive,
  title={Contrastive learning with stronger augmentations},
  author={Wang, Xiao and Qi, Guo-Jun},
  journal={IEEE transactions on pattern analysis and machine intelligence},
  volume={45},
  number={5},
  pages={5549--5560},
  year={2022},
  publisher={IEEE}
}

@inproceedings{Amir2016NTU60,
  author    = {S. Amir and J. Liu and T.-T. Ng and G. Wang},
  title     = {NTU RGB+D: A Large Scale Dataset for 3D Human Activity Analysis},
  booktitle = {Proceedings of the IEEE Conference on Computer Vision and Pattern Recognition (CVPR)},
  pages     = {1010--1019},
  year      = {2016}
}

@article{Liu2017PKUMMD,
  author    = {C. Liu and Y. Hu and Y. Li and S. Song and J. Liu},
  title     = {PKU-MMD: A Large Scale Benchmark for Continuous Multi-Modal Human Action Understanding},
  journal   = {arXiv preprint arXiv:1703.07475},
  year      = {2017},
  doi       = {10.48550/ARXIV.1703.07475}
}

@inproceedings{Xu2021CapsuleAE,
  author    = {Z. Xu and X. Shen and Y. Wong and M. S. Kankanhalli},
  title     = {Unsupervised Motion Representation Learning with Capsule Autoencoders},
  booktitle = {Advances in Neural Information Processing Systems (NeurIPS)},
  pages     = {3205--3217},
  year      = {2021},
}

@article{Chen2025VLSkeleton,
  author    = {Y. Chen and others},
  title     = {Vision-Language Meets the Skeleton: Progressively Distillation With Cross-Modal Knowledge for 3D Action Representation Learning},
  journal   = {IEEE Transactions on Multimedia},
  volume    = {27},
  pages     = {2293--2303},
  year      = {2025},
  doi       = {10.1109/TMM.2024.3521718}
}

@article{Gao2023EfficientSTCL,
  author    = {X. Gao and Y. Yang and Y. Zhang and M. Li and J.-G. Yu and S. Du},
  title     = {Efficient Spatio-Temporal Contrastive Learning for Skeleton-Based 3-D Action Recognition},
  journal   = {IEEE Transactions on Multimedia},
  volume    = {25},
  pages     = {405--417},
  year      = {2023},
  doi       = {10.1109/TMM.2021.3127040}
}

@article{liu2024cross,
  title={Cross-model cross-stream learning for self-supervised human action recognition},
  author={Liu, Mengyuan and Liu, Hong and Guo, Tianyu},
  journal={IEEE Transactions on Human-Machine Systems},
  pages={1--10},
  year={2024}
}

@article{tu2025issm,
  title={Informative Sample Selection Model for Skeleton-based Action Recognition with Limited Training Samples},
  author={Tu, Zhigang and Zhang, Zhengbo and Gong, Jia and Yuan, Junsong and Du, Bo},
  journal={arXiv preprint arXiv:2510.25345},
  year={2025}
}

@article{yang2024view,
  title        = {View-invariant Skeleton Action Representation Learning via Motion Retargeting},
  author       = {Yang, Dapeng and Wang, Yanan and Dantcheva, Antitza and Garattoni, Lorenzo and Francesca, Gianpiero and Br{\'e}mond, Fran{\c{c}}ois},
  journal      = {International Journal of Computer Vision},
  volume       = {132},
  pages        = {2351--2366},
  year         = {2024},
  doi          = {10.1007/s11263-023-01967-8}
}

@article{kakutani1941,
  title={A generalization of {B}rouwer's fixed point theorem},
  author={Kakutani, Shizuo},
  journal={Duke Mathematical Journal},
  volume={8},
  number={3},
  pages={457--459},
  year={1941}
}

@article{liu2019ntu120,
  title={Ntu rgb+ d 120: A large-scale benchmark for 3d human activity understanding},
  author={Liu, Jun and Shahroudy, Amir and Perez, Mauricio and Wang, Gang and Duan, Ling-Yu and Kot, Alex C},
  journal={IEEE transactions on pattern analysis and machine intelligence},
  volume={42},
  number={10},
  pages={2684--2701},
  year={2019},
  publisher={IEEE}
}

@inproceedings{yan2018spatial,
  title={Spatial temporal graph convolutional networks for skeleton-based action recognition},
  author={Yan, Sijie and Xiong, Yuanjun and Lin, Dahua},
  booktitle={Proceedings of the AAAI conference on artificial intelligence},
  volume={32},
  number={1},
  year={2018}
}

@inproceedings{zhang2023hierarchical,
  title={Hierarchical consistent contrastive learning for skeleton-based action recognition with growing augmentations},
  author={Zhang, Jiahang and Lin, Lilang and Liu, Jiaying},
  booktitle={Proceedings of the AAAI Conference on Artificial Intelligence},
  volume={37},
  number={3},
  pages={3427--3435},
  year={2023}
}
%


\section{Biography Section}
\begin{IEEEbiography}[{\includegraphics[width=1in,height=1.25in,clip,keepaspectratio]{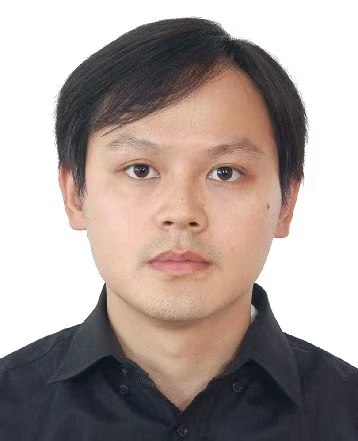}}]{Yanshan Li} received the Ph.D. degree in the South China University of Technology. He is currently a Researcher and Doctoral Supervisor with the Institute of Intelligent Information Processing and Guangdong Key Laboratory of
Intelligent Information Processing, Shenzhen University. His research interests include computer vision, machine learning, and image analysis.
\end{IEEEbiography}

\begin{IEEEbiography}[{\includegraphics[width=1in,height=1.25in,clip,keepaspectratio]{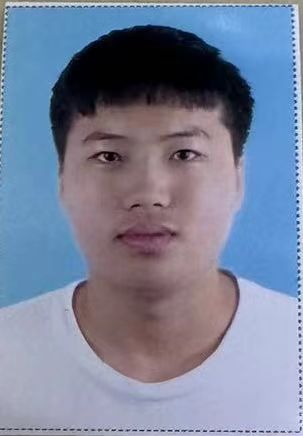}}]{Ke Ma} is currently working toward the master’s student with Shenzhen University. His research focus is on computer vision.
\end{IEEEbiography}

\begin{IEEEbiography}[{\includegraphics[width=1in,height=1.25in,clip,keepaspectratio]{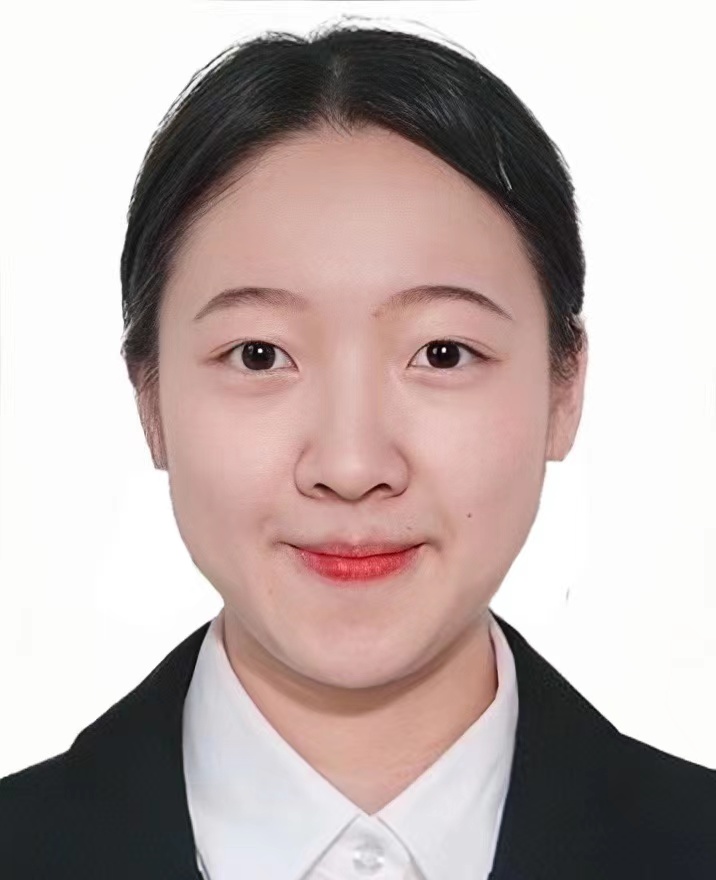}}]{Miaomiao Wei} received her master’s degree from Shenzhen University. Her research mainly focuses on human action recognition.
\end{IEEEbiography}

\begin{IEEEbiography}[{\includegraphics[width=1in,height=1.25in,clip,keepaspectratio]{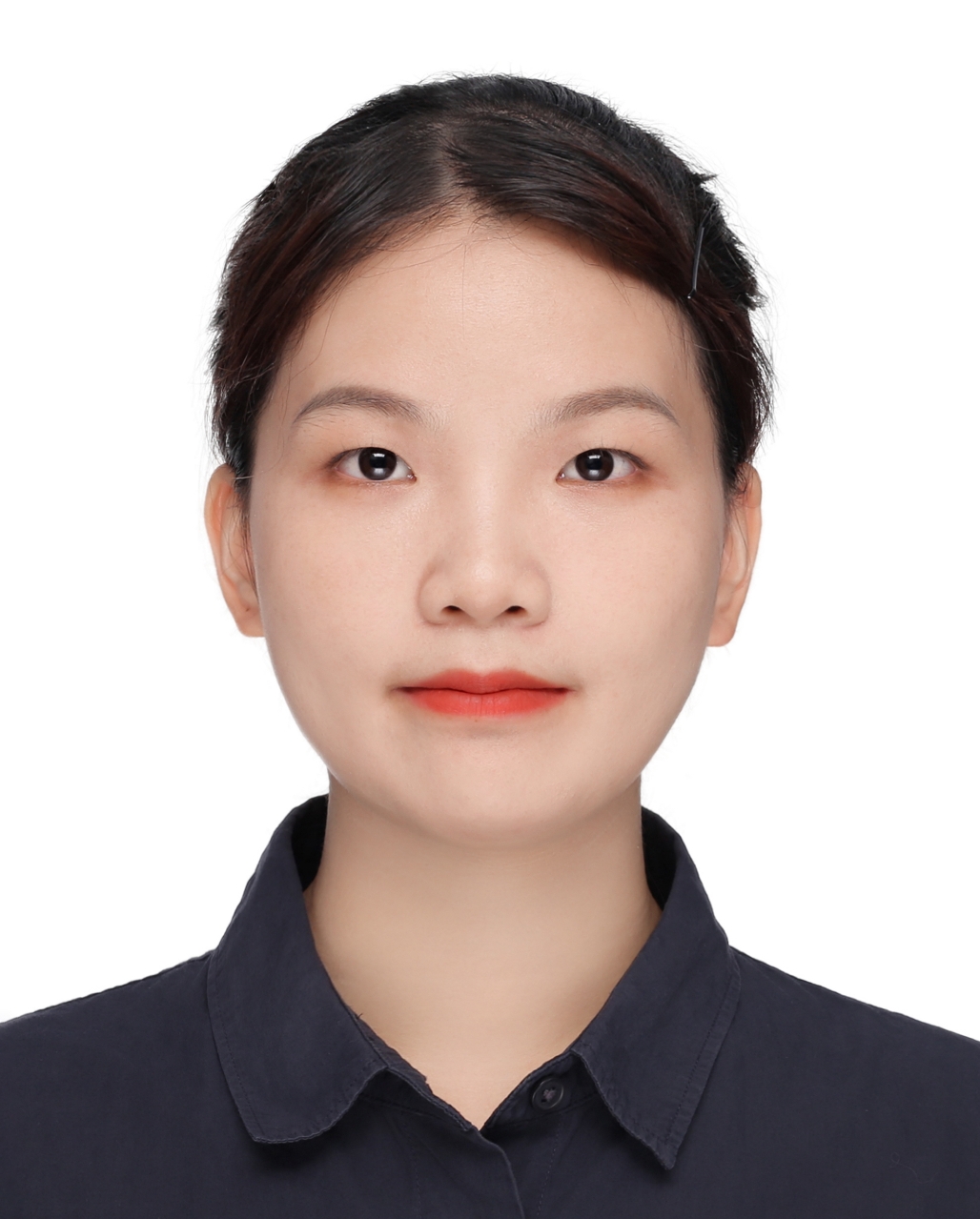}}]{Linhui Dai} (corresponding author) received the Ph.D. degree from the School of Computer Science, Peking University (PKU), China, in 2024, under the supervision of Prof. Hong Liu. She currently serves as an Assistant Professor at the college of Electronics and Information Engineering, Shenzhen University, Shenzhen, China. She has authored or coauthored in PR, T-CSVT, CAAI TRIT, and et al. Her research interests include open world object detection, underwater object detection, and salient object detection. (Email: dailinhui@szu.edu.cn)
\end{IEEEbiography}





\vfill


\end{document}